\documentclass[lettersize,journal]{IEEEtran}

\usepackage[OT1]{fontenc} 
\usepackage{adjustbox}
\usepackage{algorithm}
\usepackage{algpseudocode} 
\usepackage{amsfonts}       
\usepackage{amsmath}
\usepackage{amssymb}
\usepackage{amsthm}
\usepackage{arydshln}
\usepackage{balance}
\usepackage{bbm}
\usepackage{bm}
\usepackage{booktabs}       
\usepackage{array}
\usepackage[font=small]{caption}
\usepackage{color}
\usepackage{cite}

\usepackage{enumitem} 
\usepackage{float}
\usepackage{grffile} 
\usepackage{graphicx}
\usepackage{graphics}
\usepackage{microtype}      
\usepackage{multirow}
\usepackage{url}            
\usepackage[table]{xcolor} 
\usepackage{xparse}
\usepackage{colortbl}
\usepackage{threeparttable}
\usepackage{mhchem}
\usepackage{dsfont} 
\usepackage{subcaption}

\usepackage{hyperref}       

\newcommand{\fail}[1]{\textcolor{red}{#1}}

\newtheorem{theorem}{Theorem}

\newtheorem{corollary}{Corollary}
\theoremstyle{definition}

\newtheorem{remark}{Remark}

\usepackage{shortex}

\graphicspath{{images/}} 

\allowdisplaybreaks 

\begin{document}
	
	\title{Enhancing Diffusion-Based Quantitatively Controllable Image Generation via Matrix-Form EDM and Adaptive Vicinal Training}

	\author{Xin~Ding,~\IEEEmembership{Member,~IEEE,}
    		Yun Chen,
    		Sen Zhang,~\IEEEmembership{Member,~IEEE,}
    		Kao Zhang,
    		Nenglun Chen, 
    		Peibei Cao, \\
    		Yongwei Wang,~\IEEEmembership{Member,~IEEE,}
    		and Fei Wu,~\IEEEmembership{Senior Member,~IEEE}
	}
	\markboth{Journal of \LaTeX\ Class Files,~Vol.~14, No.~8, August~2021}%
	{Shell \MakeLowercase{\textit{et al.}}: A Sample Article Using IEEEtran.cls for IEEE Journals}
	
	
	\maketitle

	\begin{abstract}
		
        \textit{Continuous Conditional Diffusion Model} (CCDM) is a diffusion-based framework designed to generate high-quality images conditioned on continuous regression labels. Although CCDM has demonstrated clear advantages over prior approaches across a range of datasets, it still exhibits notable limitations and has recently been surpassed by a GAN-based method, namely CcGAN-AVAR. These limitations mainly arise from its reliance on an outdated diffusion framework and its low sampling efficiency due to long sampling trajectories. To address these issues, we propose an improved CCDM framework, termed iCCDM, which incorporates the more advanced \textit{Elucidated Diffusion Model} (EDM) framework with substantial modifications to improve both generation quality and sampling efficiency. Specifically, iCCDM introduces a novel matrix-form EDM formulation together with an adaptive vicinal training strategy. Extensive experiments on four benchmark datasets, spanning image resolutions from $64\times64$ to $256\times256$, demonstrate that iCCDM consistently outperforms existing methods, including state-of-the-art large-scale text-to-image diffusion models (e.g., Stable Diffusion 3, FLUX.1, and Qwen-Image), achieving higher generation quality while significantly reducing sampling cost.

	\end{abstract}
	
	
	\section{Introduction}

    \IEEEPARstart{C}{onditional} \textit{Diffusion Models} (CDMs) aim to learn the conditional probability distribution $p_\text{data}(\bm{x}|y)$ of high-dimensional data $\bm{x}$ (e.g., images) given auxiliary information $y$, thereby enabling explicit control over the generation process~\cite{croitoru2023diffusion, yang2023diffusion, ho2021classifier, 11249441, 11275890, 11249427}. Among existing CDMs, the most prominent examples are text-to-image diffusion models---such as Stable Diffusion (SD)~\cite{rombach2022high, esser2024scaling}, FLUX.1~\cite{labs2025flux1kontextflowmatching, flux2024, flux-2-2025}, and Qwen-Image~\cite{wu2025qwen}---which condition image synthesis on textual prompts. However, as shown in Figs.~\ref{fig:scatter_sfid_vs_speed_sa64} and~\ref{fig:example_imgs_real_and_fake}, these models encounter fundamental challenges when applied to \textbf{continuous quantitative variables} (i.e., \textbf{regression labels}), such as angles, ages, and temperatures. This setting, commonly referred to as \textit{Continuous Conditional Generative Modeling} (CCGM), arises in a wide range of downstream applications, including engineering inverse design~\cite{heyrani2021pcdgan, zhao2024ccdpm, fang2023diverse}, hyperspectral image processing~\cite{zhu2023data}, point cloud synthesis~\cite{triess2022point}, remote sensing image analysis~\cite{giry2022sar,shi2025regression}, model compression~\cite{ding2023distilling}, multimode sensing~\cite{chen2025hybrid}, and beyond.
    
    \begin{figure}[!htp] 
    	\centering
    	\includegraphics[width=0.9\linewidth]{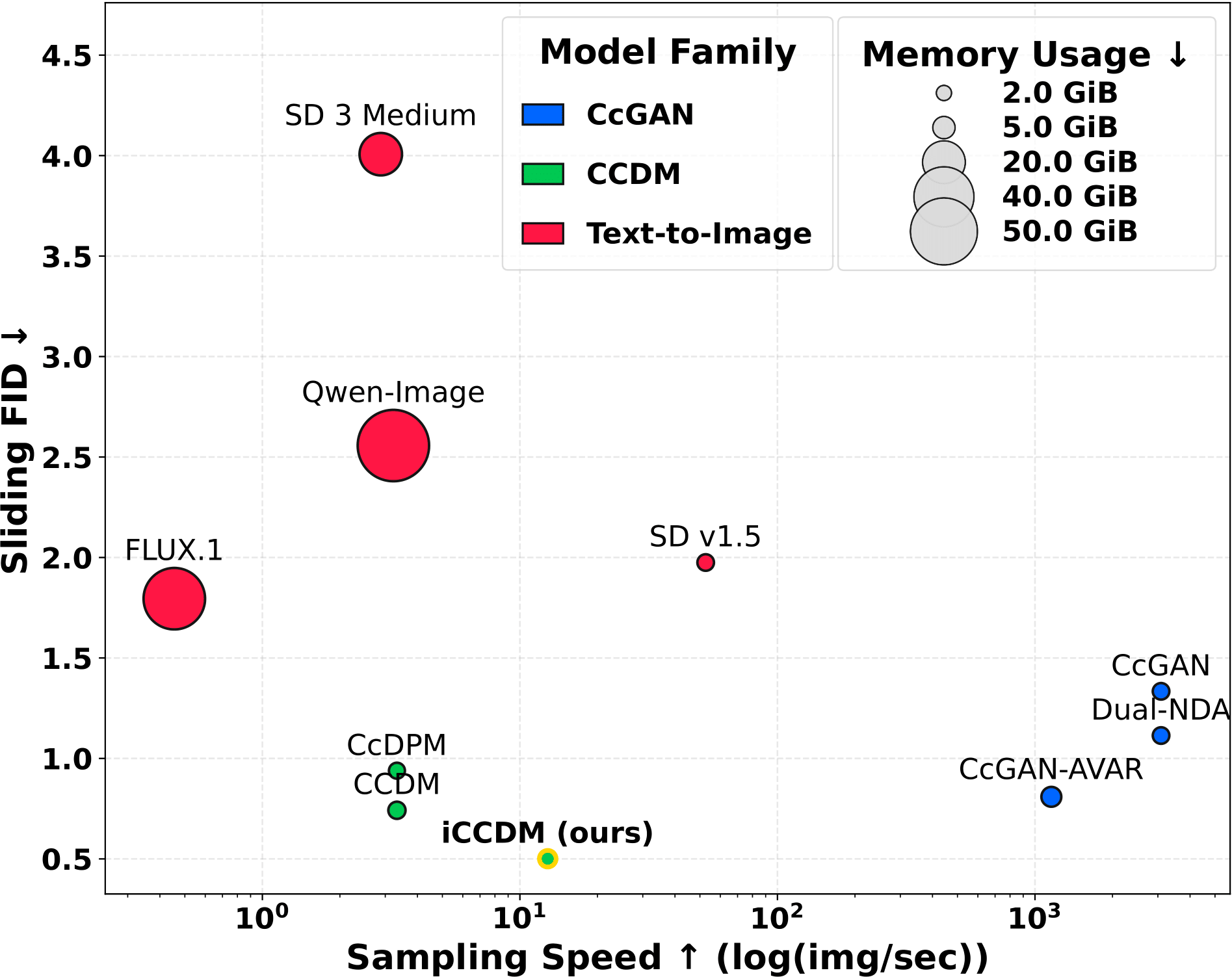}
    	\caption{ Comparison of Sliding FID versus Sampling Speed Across Three Model Families on the Steering Angle Dataset ($64\times64$). The size of each scatter point represents GPU memory usage during sampling. Arrows ($\downarrow$ or $\uparrow$) indicate whether lower or higher values are preferred. Text-to-image diffusion models are fine-tuned from officially released checkpoints using either full fine-tuning or LoRA~\cite{hu2022lora}, while all other methods are trained from scratch. }
    	\label{fig:scatter_sfid_vs_speed_sa64}
    \end{figure}
    
    To address these challenges, we previously proposed the \textit{Continuous Conditional Diffusion Model} (CCDM)~\cite{ding2025ccdm}, a diffusion-based framework specifically designed for CCGM tasks involving image data. Built upon the classical \textit{Denoising Diffusion Probabilistic Model} (DDPM) framework~\cite{ho2020denoising}, CCDM introduces a $y$-dependent diffusion process to better estimate the conditional distribution $p(\bm{x}| y)$ with regression labels. In addition, CCDM incorporates the vicinal training strategy originally developed for \textit{Continuous Conditional Generative Adversarial Networks} (CcGANs)~\cite{ding2021ccgan, ding2023ccgan}. Extensive empirical results have demonstrated the superiority of CCDM over prior approaches across a variety of benchmark datasets.

    Despite its effectiveness, CCDM exhibits several notable limitations. First, its performance has recently been surpassed by CcGAN-AVAR~\cite{ding2025imbalance}, a GAN-based method. Specifically, CCDM often suffers from severe label inconsistency---where generated samples fail to match the given regression labels---on challenging datasets such as Steering Angle ($256 \times 256$). These issues may be attributed to inherent constraints within the DDPM framework~\cite{ho2020denoising}, the standard UNet architecture~\cite{ronneberger2015u}, and the fixed-width vicinity mechanism~\cite{ding2023ccgan}. Second, CCDM suffers from low sampling efficiency, typically requiring long trajectories (often exceeding 150 steps) to achieve satisfactory generation quality. Third, to construct $y$-dependent diffusion processes, CCDM relies on a large pre-trained covariance embedding network implemented as a multi-layer perceptron (MLP), which maps a scalar regression label $y$ to a high-dimensional embedding vector $\bm{h}^l(y)$ matching the dimensionality of $\bm{x}$. This design incurs substantial memory overhead during sampling.

	Motivated by these observations, we propose the \textit{improved Continuous Conditional Diffusion Model} (iCCDM), which extends the \textit{Elucidated Diffusion Model} (EDM)~\cite{karras2022elucidating}---one of the most influential modern diffusion frameworks---to the CCGM setting through a series of \textbf{in-depth modifications}. Our main contributions are summarized as follows:
    \begin{itemize}
        \item We extend the EDM framework for CCGM by explicitly incorporating conditioning information into the noise perturbation process, enabling condition-aware diffusion.
        
        \item We formulate matrix-form forward and reverse \textit{Stochastic Differential Equations} (SDE) and the corresponding \textit{Probability Flow Ordinary Differential Equation} (PF-ODE), and introduce a non-negative weighting coefficient $\lambda_y$ to flexibly modulate the influence of the condition---an ability absent from the original CCDM.
        
        \item We propose a novel vicinal score estimate by refining adaptive vicinal techniques, and adapt the EDM preconditioning, network architecture, training, and sampling procedures to the matrix-form setting.
        
        \item We improve efficiency by replacing the covariance embedding network of CCDM with a lightweight five-layer Convolutional Neural Network (CNN), significantly reducing model parameters during training and sampling.
        
        \item Extensive experiments on four benchmark datasets at resolutions ranging from $64\times64$ to $256\times256$ demonstrate that iCCDM achieves state-of-the-art performance.
    \end{itemize}

	\section{Preliminary}
	
	\subsection{Continuous Conditional Generative Modeling}\label{sec:related_ccgm}

    \textit{Continuous Conditional Generative Adversarial Networks} (CcGANs)~\cite{ding2021ccgan, ding2023ccgan, ding2024turning}, introduced by Ding \emph{et al.}, represent a pioneering effort to address CCGM tasks within the GAN framework. To mitigate data scarcity inherent in CCGM settings, CcGANs employ vicinal discriminator losses, leveraging samples within a fixed hard or soft vicinity of a target label $y$ to estimate the conditional distribution $p(\bm{x}\mid y)$. Furthermore, CcGANs utilize an \textit{Improved Label Input} (ILI) mechanism, which maps the scalar label $y$ to a $128\times1$ embedding vector $\bm{h}_y^s$ before feeding it into the neural networks.

    As an alternative to CcGANs, Ding \emph{et al.}~\cite{ding2025ccdm} later introduced the \textit{Continuous Conditional Diffusion Model} (CCDM), the first diffusion-based framework for CCGM tasks. Built upon the DDPM architecture, CCDM derives $y$-dependent diffusion processes and integrates a hard vicinal image denoising loss. While extensive experiments demonstrate that CCDM yields superior generation quality compared to CcGANs~\cite{ding2021ccgan, ding2023ccgan, ding2024turning}, it has notable limitations. As shown in Section~\ref{sec:experiment}, CCDM suffers from significant label inconsistency on challenging datasets, such as Steering Angle ($256\times256$). Additionally, its long sampling trajectories result in low sampling efficiency. Although distillation-based sampling~\cite{yin2024improved} can improve sampling speed, it comes at the cost of reduced image quality, as noted in~\cite{ding2025ccdm}. Moreover, the reliance on a heavyweight multi-layer perceptron to construct $y$-dependent diffusion processes incurs substantial memory overhead during sampling.
    
    Most recently, Ding \emph{et al.} introduced CcGAN-AVAR~\cite{ding2025imbalance}, a comprehensive reformulation of the original CcGAN framework featuring a refined network architecture, novel adaptive vicinity mechanisms, and auxiliary GAN training regularization. In contrast to the fixed vicinity of vanilla CcGANs, CcGAN-AVAR employs \textit{Adaptive Vicinity} (AV) strategies---specifically Soft AV and Hybrid AV---which dynamically adjust weight decay rates according to local sample density. These adaptive mechanisms significantly enhance label consistency, enabling CcGAN-AVAR to outperform both vanilla CcGANs and CCDM across various benchmark datasets.

	\subsection{Elucidated Diffusion Model}\label{sec:related_edm}
    
    The \textit{Elucidated Diffusion Model} (EDM), proposed by Karras \emph{et al.}~\cite{karras2022elucidating}, provides a principled reformulation of diffusion models that unifies DDPMs~\cite{ho2020denoising,nichol2021improved} and \textit{Score-based Generative Models} (SGMs)~\cite{song2021scorebased}. Unlike DDPMs, which rely on a discretized forward process with a fixed noise schedule, and SGMs, which are typically formulated via continuous-time score matching, EDM adopts a noise-conditioned parameterization that decouples training and sampling from any specific forward process. Noise is injected through simple additive Gaussian corruption with a continuously sampled noise level, enabling the model to learn denoising behavior across a wide range of noise scales. To support this, EDM employs a preconditioning strategy that rescales the network input, output, and skip connections as explicit functions of the noise level, ensuring stable signal magnitudes throughout the network and decoupling the optimization dynamics from the noise scale. Training is therefore carried out using a noise-weighted denoising objective that balances contributions across noise levels and improves training stability.

    From a stochastic process perspective, EDM corresponds to a variance-exploding forward SDE and its associated reverse-time SDE. Equivalently, sampling can be performed by solving the corresponding probability flow ODE~\cite{chen2023probability}, which shares the same marginal distributions as the SDE while yielding a deterministic trajectory. This formulation enables efficient, high-quality generation using higher order numerical solvers, substantially reducing the sampling steps required compared to DDPMs while maintaining or improving sample fidelity.

	\section{Method} \label{sec:method}
	
     To enhance both the generation performance and sampling efficiency of CCDM~\cite{ding2025ccdm}, we migrate from the DDPM framework to the more advanced EDM framework, proposing an improved CCDM formulation. Under this new paradigm, we reformulate the noise perturbation process, the forward and backward SDEs, the probability flow ODE, and the associated training and sampling procedures.
    
	\subsection{Perturbation Process with Condition-Specific Noise}\label{sec:noise_perturnation}
	
	Let $\bm{x}\in \mathbb{R}^d$ be a sample from the real data distribution conditioned on a regression label $y\in \mathbb{R}$, denoted as $p_\text{data}(\bm{x}|y)$, with a simplified covariance matrix $\bm{\Sigma}_\text{data}=\sigma_\text{data}^2\bm{I}$. While EDM~\cite{karras2022elucidating} employs a condition-agnostic noising process to perturb real data, we introduce a condition-specific forward perturbation to enable finer control of generated samples under regression labels. This process is defined as:
	\[
		\bm{x}_t = \bm{x}_0 + \bm{\Sigma}(t,y)^{\frac{1}{2}}\bm{\varepsilon},
		\label{eq:iccdm_noise_process}
	\]
	where $\bm{x}_0\in\mathbb{R}^d$ is a clean sample drawn from $p_\text{data}(\bm{x}|y)$, $\bm{\varepsilon}\in\mathbb{R}^d$ is Gaussian white noise sampled from $\mathcal{N}(\bm{0},\bm{I}_{d\times d})$, and $\bm{x}_t\in\mathbb{R}^d$ is the noisy sample at time $t\in [0,T]$. Here, $\bm{\Sigma}(t,y)^{\frac{1}{2}}\in\mathbb{R}^{d\times d}$ is a symmetric matrix square root satisfying 
	\[
		\bm{\Sigma}(t,y)^{\frac{1}{2}} \bm{\Sigma}(t,y)^{\frac{1}{2}} = \bm{\Sigma}(t,y).
		\label{eq:Sigma_square}
	\]
	The time-dependent, condition-specific noise covariance matrix $\bm{\Sigma}(t,y)\in\mathbb{R}^{d\times d}$ defines the transition distribution:
	\[
		p_{0t}(\bm{x}|\bm{x}_0,y) = \mathcal{N}(\bm{x}_0, \bm{\Sigma}(t,y)).
		\label{eq:transition_dist}
	\]
	The noisy data distribution conditional on $y$ at time $t$, denoted by $p_t(\bm{x}|y)$, can then be expressed as
	\[
		p_t(\bm{x}|y) & = \int p_{0t}(\bm{x}|\bm{x}_0,y) p_\text{data}(\bm{x}_0|y) \mathrm{d} \bm{x}_0 \\
		& = \int p_\text{data}(\bm{x}_0|y) \mathcal{N}(\bm{x} ; \bm{x}_0, \bm{\Sigma}(t,y)) \mathrm{d} \bm{x}_0 \\
		& = \left[ p_\text{data} \ast \mathcal{N}(\bm{x}_0, \bm{\Sigma}(t,y)) \right](\bm{x}|y) \\
		& \triangleq  p_{\bm{\Sigma}(t,y)}(\bm{x}|y),
		\label{eq:p_t}
	\]
	where $p \ast q$ denotes the convolution between two probability density functions $p$ and $q$. Following~\cite{karras2022elucidating}, we replace $p_t(\bm{x}|y)$ with $p_{\bm{\Sigma}}(\bm{x}|y)$, allowing direct control of the noise level by parameterizing it with $\bm{\Sigma}$ instead of $t$.
	
	
	\subsection{Condition-Specific and Matrix-Form SDEs and PF-ODE}\label{sec:sde_and_ode}

    In this section, we present the matrix-form SDEs and the corresponding PF-ODE for the condition-specific perturbation process described above.

	\subsubsection{Forward SDE}
	
	The perturbation process introduced in Eq.~\eqref{eq:iccdm_noise_process} can be expressed as a forward \textit{Stochastic Differential Equation} (SDE)~\cite{oksendal2013stochastic} formulated as
	\[
		\mathrm{d}\bm{X}_t = \dot{\bm{\Sigma}}(t,y)^{\frac{1}{2}}\mathrm{d}\bm{B}_t,
		\label{eq:iccdm_forward_sde}
	\]
	where $t$ is a continuous time variable, $\bm{B}_t$ is a standard Wiener process, $\bm{X}_0 \sim p_\text{data}(\bm{x}|y)$ is the initial state, $\bm{X}_t\in\mathbb{R}^d$ denotes the intermediate state, and $\dot{\bm{\Sigma}}(t,y)^{\frac{1}{2}}\in\mathbb{R}^{d\times d}$ is a condition-specific diffusion coefficient satisfying
	\[
		\dot{\bm{\Sigma}}(t,y)^{\frac{1}{2}}\dot{\bm{\Sigma}}(t,y)^{\frac{1}{2}} = \dot{\bm{\Sigma}}(t,y),
		\label{eq:Q_GG_Sigma_dot}
	\]
	where $\dot{\bm{\Sigma}}(t, y)$ represents the partial derivative of the noise covariance matrix $\bm{\Sigma}(t, y)$ with respect to $t$. Accordingly, $\dot{\bm{\Sigma}}(t, y)^{\frac{1}{2}}$ denotes the symmetric matrix square root of $\dot{\bm{\Sigma}}(t, y)$.
	
	\begin{theorem}
		\label{thm:noising_process}
		Given the forward diffusion process defined in Eq.~\eqref{eq:iccdm_forward_sde}, the conditional distribution of $\bm{X}_t$ given $\bm{X}_0 = \bm{x}_0$ is Gaussian:
		\[
		\bm{X}_t | \bm{X}_0 = \bm{x}_0 \sim \mathcal{N}(\bm{x}_0, \bm{\Sigma}(t,y)).
		\]
	\end{theorem}
	
	\begin{remark}
		The transition distribution defined in Eq.~\eqref{eq:transition_dist} for the perturbation process follows directly from Theorem~\ref{thm:noising_process}. The proof of Theorem~\ref{thm:noising_process} is provided in Appendix~C. 
	\end{remark}

	\subsubsection{Reverse SDE}
	
	According to Anderson's time-reversal theorem \cite{anderson1982reverse} and Song's derivation~\cite{song2021scorebased}, the reverse-time process can be formulated as another SDE:
	\[
		\mathrm{d}\bm{X}_t = &- \dot{\bm{\Sigma}}(t,y) \nabla_{\bm{x}}\log p_{\bm{\Sigma}(t,y)}(\bm{X}_t|y) \mathrm{d}t \\
		& + \dot{\bm{\Sigma}}(t,y)^{\frac{1}{2}}\mathrm{d}\bar{\bm{B}_t},
		\label{eq:iccdm_reverse_sde}
	\]
	where $\bar{\bm{B}}_t$ denotes a standard Wiener process evolving in reverse time, and $\nabla_{\tilde{\bm{x}}}\log p_{\bm{\Sigma}}(\tilde{\bm{x}}|y)$ is the \textit{conditional score function}---a vector field that points in the direction of increasing data density given the noise level $\bm{\Sigma}$ and condition $y$.
	

	\subsubsection{Probability Flow ODE}
	
	Building on the SDEs introduced above, and following the derivation of Song \emph{et al.}.~\cite{song2021scorebased}, we can remove the stochastic term to obtain the corresponding matrix-form \textit{Probability Flow ODE} (PF-ODE):
	\[
		\mathrm{d}\bm{X}_t & = - \frac{1}{2}\dot{\bm{\Sigma}}(t,y)\nabla_{\bm{x}}\log p_{\bm{\Sigma}(t,y)}(\bm{X}_t|y)\mathrm{d}t.
		\label{eq:iccdm_pf_ode}
	\]

	\subsection{Design of Diffusion Coefficient and Related Matrices}\label{sec:diffusion_coefficient}
	
	The noise perturbation (Eq.~\eqref{eq:iccdm_noise_process}), condition-specific SDEs (Eqs.~\eqref{eq:iccdm_forward_sde} and \eqref{eq:iccdm_reverse_sde}), and PF-ODE (Eq.~\eqref{eq:iccdm_pf_ode}) depend on four unspecified matrices: $\dot{\bm{\Sigma}}(t,y)^{\frac{1}{2}}$, $\dot{\bm{\Sigma}}(t,y)$, $\bm{\Sigma}(t,y)$, and $\bm{\Sigma}(t,y)^{\frac{1}{2}}$. 
	
	We define $\bm{G}(t,y)\triangleq\dot{\bm{\Sigma}}(t,y)^{\frac{1}{2}}$ as the diffusion coefficient in Eq.~\eqref{eq:iccdm_forward_sde}. We assume that $\bm{G}(t,y)$ takes a diagonal form:
	\[
	\bm{G}(t,y) = \text{Diag}\left( \left[g_i(t,y)\right]^d_{i=1} \right),
	\label{eq:diff_coeff_S}
	\]
	where $g_i(t,y)\in\mathbb{R}^+$, and $\text{Diag}(\cdot)$ constructs a diagonal matrix from a given vector. 
 
    Let $\bm{h}^l(y)=[h^l_1(y),\cdots,h^l_d(y)]^\intercal\in\mathbb{R}^d$ denote a \textbf{long positive embedding vector}, in contrast to the short embedding $\bm{h}^s(y)\in\mathbb{R}^{128}$ used in the ILI mechanism~\cite{ding2023ccgan}, where $d\gg 128$. This vector is produced by a pre-trained \textbf{covariance embedding network} $\phi(y)$ \cite{ding2025ccdm}. We transform $\bm{h}^l(y)$ into another positive embedding $\tilde{\bm{h}}^l(y)$ via the following element-wise mapping:
	\[
	\tilde{h}^l_i(y) = \exp(-h^l_i(y))\in (0,1].
	\label{eq:tilde_h_y}
	\]
	Using this embedding, we define 
	\[
		g_i(t,y)=\sqrt{2\dot{\sigma}(t)\sigma(t) + \lambda_y \tilde{h}^l_i(y)\dot{\sigma}(t) },
		\label{eq:g_i}
	\]
	where $\sigma(t)\in\mathbb{R}^+$ is the scalar noise level in the original EDM framework~\cite{karras2022elucidating}, $\dot{\sigma}(t)$ denotes its time derivative, and \textbf{$\lambda_y\geq 0$ is a non-negative weighting coefficient that controls the strength of the condition $y$ in the diffusion process}.
 
    Substituting Eq.~\eqref{eq:g_i} into Eq.~\eqref{eq:diff_coeff_S} yields
	\[
	\bm{G}(t,y) = \text{Diag}\left( \left[\sqrt{2\dot{\sigma}(t)\sigma(t) + \lambda_y \tilde{h}^l_i(y)\dot{\sigma}(t) } \right]^d_{i=1} \right).
	\label{eq:diff_coeff_S_2}
	\]
    Given this formulation, the remaining matrices can be derived sequentially as follows:
    \[
        \dot{\bm{\Sigma}}(t,y) & = \bm{G}(t,y)\bm{G}(t,y) \\
        & = \text{Diag}\left( \left[2\dot{\sigma}(t)\sigma(t) + \lambda_y \tilde{h}^l_i(y)\dot{\sigma}(t) \right]^d_{i=1} \right), \\
        \bm{\Sigma}(t,y) & = \bm{\Sigma}(0,y)+\int_0^t\bm{G}(s,y)\bm{G}(s,y)\mathrm{d}s \\
        & = \text{Diag}\left(\left[ \sigma(t)^2 + \lambda_y \tilde{h}^l_i(y)\sigma(t) \right]^d_{i=1} \right),\label{eq:Sigma} \\
        \bm{\Sigma}(t,y)^{\frac{1}{2}} & = \text{Diag}\left(\left[\sqrt{ \sigma(t)^2 + \lambda_y \tilde{h}^l_i(y)\sigma(t) } \right]^d_{i=1} \right), 
    \]
    where the initial value $\bm{\Sigma}(0,y)$ is set to $\bm{0}$. It is easy to verify that all four matrices are symmetric and positive definite.
    
	
	\begin{remark}
        \label{rmk:sigma_t}
        Following \cite{karras2022elucidating}, $\sigma$ is treated as an independent variable and sampled from a log-normal distribution, $\log(\sigma) \sim \mathcal{N}(P_{\text{mean}}, P_{\text{std}}^{2})$ in training. During sampling, we set $\sigma(t) = t$.
     \end{remark}

     \begin{remark}
        \label{rmk:lambda_y}

        Unlike \cite{ding2025ccdm}, which directly defines the diagonal entries of the noise covariance matrix $\bm{\Sigma}(t,y)$ as the embedding vector $\bm{h}^l(y)$, we introduce a weighting coefficient $\lambda_y$ that enables flexible control over the strength of the conditioning variable $y$ in the diffusion process across different datasets.

        Setting $\lambda_y = 0$ and removing the vicinal weights defined in Section~\ref{sec:training_process} reduces the proposed matrix-form EDM framework to the original EDM formulation.
        
        
     \end{remark}
	
	\subsection{Vicinal Training and Network Design}\label{sec:training_process}
	
	The reverse SDE and PF-ODE defined in Eqs.~\eqref{eq:iccdm_reverse_sde} and \eqref{eq:iccdm_pf_ode}, respectively, enable sampling from the data distribution $p_\text{data}(\bm{x}|y)$, provided that the conditional score function $\nabla_{\tilde{\bm{x}}}\log p_{\bm{\Sigma}}(\tilde{\bm{x}}|y)$, or a sufficiently accurate approximation thereof, is available. In this section, we present an estimate for the conditional score function, together with a corresponding training procedure that incorporates matrix conditioning~\cite{karras2022elucidating} and adaptive vicinity~\cite{ding2025imbalance}. We further discuss the network architecture design of both the score estimate and the covariance embedding network.
    
	\subsubsection{Vicinal Score Estimate}\label{sec:estimate_score}

    Following~\cite{karras2022elucidating}, we assume the existence of a denoiser function $D_{\bm{\theta}}(\tilde{\bm{x}},y,\bm{\Sigma})$, which takes as input a noisy sample $\tilde{\bm{x}}$, its corresponding label $y$, and the noise level $\bm{\Sigma}$, and outputs a denoised estimate $\hat{\bm{x}}$. Here, $\bm{\theta}$ denotes the learnable parameters of $D$. A common objective for training $D$, given fixed $\bm{\Sigma}$ and $y$, is the $L_2$ loss defined as
	\[
		& \mathcal{L}(D;y,\bm{\Sigma}) \\
		= &\mathbb{E}_{\bm{X}\sim p_\text{data}(\bm{x}|y), \bm{\varepsilon}\sim\mathcal{N}(\bm{0},\bm{\Sigma})}\left[ \left\| D_{\bm{\theta}}(\tilde{\bm{X}},y,\bm{\Sigma}) - \bm{X} \right\|_2^2  \right],
		\label{eq:L_D_fixed_y_Sigma}
	\]
	where the noisy sample $\tilde{\bm{X}} = \bm{X} + \bm{\varepsilon}\sim p_{\bm{\Sigma}}(\tilde{\bm{x}}|y)$ is a random variable, as are $\bm{X}$ and $\bm{\varepsilon}$.
	As shown below, the denoiser $D$ can be used to approximate the score function $\nabla_{\tilde{\bm{x}}}\log p_{\bm{\Sigma}}(\tilde{\bm{x}}|y)$.
	
	As an initial step in linking the denoiser to the score function, Karras et~al.~\cite{karras2022elucidating} approximate $p_\text{data}(\bm{x}|y)$ using a finite set of samples via the empirical probability density defined as
	\[
		\hat{p}^\delta_\text{data}(\bm{x}|y) = \frac{1}{N^y}\sum_{i=1}^{N^y}\delta(\bm{x}-\bm{x}_{0,i}^y),
		\label{eq:p_data_emp_est}
	\]
	where $y$ is the target conditioning regression label, $\bm{x}_{0,i}^y$ denotes the $i$-th clean sample in the training set associated with label $y$, $N^y$ is the number of training samples corresponding to $y$, and $\delta(\cdot)$ is the Dirac delta function. However, as pointed out by Ding et~al.~\cite{ding2023ccgan}, Eq.~\eqref{eq:p_data_emp_est} provides a poor approximation of the conditional data distribution $p_\text{data}(\bm{x}|y)$. The key issue lies in the sparsity of regression datasets, where the number of samples corresponding to a specific label value is often very small---or even zero---rendering $\hat{p}^\delta_\text{data}(\bm{x}|y)$ ineffective. 
	
	To overcome this limitation, we define a vicinal estimate to approximate $p_\text{data}(\bm{x}|y)$:
	\[
		\hat{p}^v_\text{data}(\bm{x}|y) = C_y\sum_{i=1}^{N}W_{i,y}\delta(\bm{x}-\bm{x}_{0,i}),
		\label{eq:p_data_vic_est}
	\]
	where $\bm{x}_{0,i}$ is the $i$-th clean sample in the training set (whose label may not exactly equal $y$), and $W_{i,y}$ denotes the \textit{vicinal weight} associated with $\bm{x}_{0,i}$. The weight $W_{i,y}$ can be defined using either fixed-size vicinal weights~\cite{ding2023ccgan} or adaptive vicinal weights~\cite{ding2025imbalance}.
	The normalization constant $C_y$ ensures that $\hat{p}^v_\text{data}(\bm{x}|y)$ forms a valid probability density function. This vicinal estimate enables us to uses samples with labels falling within a vicinity of $y$ to better estimate $p_\text{data}(\bm{x}|y)$.
	
	Next, we reorganize the denoiser function’s loss in Eq.~\eqref{eq:L_D_fixed_y_Sigma} and replace $p_\text{data}$ with $\hat{p}^v_\text{data}$:
	\begingroup
	\[
		& \mathcal{L}(D;y,\bm{\Sigma}) \\
		= &\mathbb{E}_{\bm{X}\sim p_\text{data}(\bm{x}|y)}\mathbb{E}_{\tilde{\bm{X}}\sim\mathcal{N}(\bm{X},\bm{\Sigma})}\left[ \left\| D(\tilde{\bm{X}},y,\bm{\Sigma}) - \bm{X} \right\|_2^2  \right]\\
		\approx & C_y \int \underbrace{\left[ \sum_{i=1}^{N}  W_{i,y} \mathcal{N}(\tilde{\bm{x}}; \bm{x}_{0,i}, \bm{\Sigma}) \left\| D(\tilde{\bm{x}},y,\bm{\Sigma}) - \bm{x}_{0,i} \right\|_2^2 \right]}_{\triangleq \hat{\mathcal{L}}^v(D; \tilde{\bm{x}}, y,\bm{\Sigma})} \mathrm{d} \tilde{\bm{x}} \\
		\triangleq & \hat{\mathcal{L}}^v(D;y,\bm{\Sigma})	\label{eq:vicinal_denoiser_loss}	
	\]
	\endgroup
	Following~\cite{karras2022elucidating}, minimizing $\hat{\mathcal{L}}^v(D;y,\bm{\Sigma})$ in Eq.~\eqref{eq:vicinal_denoiser_loss} reduces to independently minimizing $\hat{\mathcal{L}}^v(D; \tilde{\bm{x}}, y,\bm{\Sigma})$ for each $\tilde{\bm{x}}$:
	\[
		D^\ast(\tilde{\bm{x}},y,\bm{\Sigma}) = \arg\min_{D}\hat{\mathcal{L}}^v(D; \tilde{\bm{x}}, y,\bm{\Sigma}),
		\label{eq:minimization_D_ast}
	\]
	where $D^\ast$ is the optimal denoiser function that minimizes $\hat{\mathcal{L}}^v(D; \tilde{\bm{x}}, y,\bm{\Sigma})$ and admits the closed-form solution
    \[
        D^\ast(\tilde{\bm{x}},y,\bm{\Sigma}) = \frac{\sum_{i=1}^{N} W_{i,y}\mathcal{N}(\tilde{\bm{x}}; \bm{x}_{0,i}, \bm{\Sigma})\bm{x}_{0,i}}{\sum_{i=1}^{N} W_{i,y}\mathcal{N}(\tilde{\bm{x}}; \bm{x}_{0,i}, \bm{\Sigma})}.
		\label{eq:D_star}
    \]

	In Eq.~\eqref{eq:p_t}, we have shown that the noisy data distribution $p_{\bm{\Sigma}}(\tilde{\bm{x}}|y)$ can be expressed as the convolution of $p_\text{data}(\bm{x}|y)$ and a Gaussian $\mathcal{N}(\bm{x}_0, \bm{\Sigma})$. By substituting Eq.~\eqref{eq:p_data_vic_est} into Eq.~\eqref{eq:p_t}, we can obtain an estimate of $p_{\bm{\Sigma}}(\tilde{\bm{x}}|y)$:
	\[
	p_{\bm{\Sigma}}(\tilde{\bm{x}}|y) = & \left[ p_\text{data} \ast \mathcal{N}(\bm{x}_0, \bm{\Sigma}) \right](\tilde{\bm{x}}|y) \\
	\approx & \int \hat{p}^v_\text{data}(\bm{x}_0|y) \mathcal{N}(\tilde{\bm{x}} ; \bm{x}_0, \bm{\Sigma}) \mathrm{d} \bm{x}_0 \\
	= & C_y\sum_{i=1}^{N} W_{i,y} \mathcal{N}(\tilde{\bm{x}} ; \bm{x}_{0,i}, \bm{\Sigma})  \triangleq \hat{p}^v_{\bm{\Sigma}}(\tilde{\bm{x}}|y). \label{eq:p_sigma_vic_est}
	\]
	With Eq.~\eqref{eq:p_sigma_vic_est}, the gradient of $\log \hat{p}^v_{\bm{\Sigma}}(\tilde{\bm{x}}|y)$ with respect to $\tilde{\bm{x}}$ provides an estimate of $\nabla_{\tilde{\bm{x}}}\log p_{\bm{\Sigma}}(\tilde{\bm{x}}|y)$ and takes the form
	\[
	& \nabla_{\tilde{\bm{x}}}\log\hat{p}^v_{\bm{\Sigma}}(\tilde{\bm{x}}|y) \\
	= & \bm{\Sigma}^{-1}\left( \frac{\sum_{i=1}^{N} W_{i,y} \mathcal{N}(\tilde{\bm{x}} ; \bm{x}_{0,i}, \bm{\Sigma})\bm{x}_{0,i}}{\sum_{i=1}^{N} W_{i,y} \mathcal{N}(\tilde{\bm{x}} ; \bm{x}_{0,i}, \bm{\Sigma})} - \tilde{\bm{x}}  \right).
	\label{eq:grad_p_sigma_vic_est}
	\]
	By substituting Eq.~\eqref{eq:D_star} into Eq.~\eqref{eq:grad_p_sigma_vic_est}, we obtain an explicit relationship between the score function and the denoiser:
	\[
		\nabla_{\tilde{\bm{x}}}\log\hat{p}^v_{\bm{\Sigma}}(\tilde{\bm{x}}|y) & = \bm{\Sigma}^{-1}\left( D^\ast(\tilde{\bm{x}},y,\bm{\Sigma}) - \tilde{\bm{x}} \right) \\
		& \approx\bm{\Sigma}^{-1}\left( D_{\bm{\theta}}(\tilde{\bm{x}},y,\bm{\Sigma}) - \tilde{\bm{x}} \right).
		\label{eq:score_by_D}
	\]
	Eq.~\eqref{eq:score_by_D} thus provides a vicinal estimate of the conditional score function $\nabla_{\tilde{\bm{x}}}\log p_{\bm{\Sigma}}(\tilde{\bm{x}}|y)$ expressed in terms of the learned denoiser $D_{\bm{\theta}}$.

    \begin{remark}

        The vicinal weight $W_{i,y}$ for $\bm{x}_{0,i}$ in Eq.~\eqref{eq:p_data_vic_est} is defined as an indicator function: $W_{i,y} = \mathds{1}_{\{|y-y_i|\leq\kappa_{y}\}}$, where $y_i$ denotes the regression label associated with $\bm{x}_{0,i}$, and $\kappa_{y}$ represents the radius of the closed interval $[y-\kappa_{y}, y+\kappa_{y}]$ centered at $y$. The radius $\kappa_{y}$ is computed using Algorithm~1 in~\cite{ding2025imbalance}, and it is dynamically adjusted based on the local sample count $N^y$ at $y$ and a hyperparameter $N_\text{AV}$ that specifies the minimum effective number of samples per vicinity. This interval, centered at $y$, is referred to as the \textit{Hard Adaptive Vicinity} (Hard AV). In this work, we adopt the Hard AV approach rather than the \textit{Soft or Hybrid Adaptive Vicinities} (Soft/Hybrid AV) proposed in~\cite{ding2025imbalance}, as Hard AV is more robust in our experiments, particularly in challenging settings such as Steering Angle, as shown in Table~\ref{tab:ab_effect_vicinity_type}.
         
    \end{remark}

	\subsubsection{Matrix Preconditioning}\label{sec:matrix_preconditioning}
    
    We apply preconditioning to the denoiser $D_{\bm{\theta}}$ as
	\[
		D_{\bm{\theta}}(\tilde{\bm{x}},y,\bm{\Sigma}) = C_\text{skip}^{\bm{\Sigma}}\tilde{\bm{x}} + C_\text{out}^{\bm{\Sigma}} F_{\bm{\theta}}(C_\text{in}^{\bm{\Sigma}} \tilde{\bm{x}}, y, C_\text{noise}^{\bm{\Sigma}}),
		\label{eq:preconditioning}
	\]
	where $F_{\bm{\theta}}$ denotes the neural network to be trained. The preconditioning coefficients of the EDM framework are extended to the matrix-form setting as
	\[
	& C_\text{in}^{\bm{\Sigma}} = (\sigma_\text{data}^2\bm{I}+\bm{\Sigma})^{-\frac{1}{2}}, \quad C_\text{skip}^{\bm{\Sigma}} = \sigma_\text{data}^2(\sigma_\text{data}^2\bm{I}+\bm{\Sigma})^{-1}, \\
	& C_\text{out}^{\bm{\Sigma}} = \sigma_\text{data}\bm{\Sigma}^{\frac{1}{2}}(\sigma_\text{data}^2\bm{I}+\bm{\Sigma})^{-\frac{1}{2}}, \quad C_\text{noise}^{\bm{\Sigma}} = \frac{1}{4}\log(\bm{\Sigma}).
     \label{eq:preconditioning}
	\]

	\subsubsection{Vicinal Training Loss}\label{sec:overall_loss}
    
    Before deriving the final training objective for $D_{\bm{\theta}}$, we define a vicinal estimate of the joint data distribution $p_\text{data}(\bm{x},y)$ as
	\begingroup
	\[
		\hat{p}^v_\text{data}(\bm{x},y) & \triangleq \hat{p}^\text{KDE}_\text{data}(y) \cdot \hat{p}^v_\text{data}(\bm{x}|y) \\
		= & \left[ \frac{1}{N}\sum_{i=1}^N\exp\left( -\frac{(y-y_i)^2}{2\sigma_\text{KDE}^2} \right) \right] \\
		& \cdot \left[ C_y\sum_{i=1}^{N}W_{i,y}\delta(\bm{x}-\bm{x}_{0,i})  \right],
        \label{eq:joint_estimate_w_kde_vicinal}
	\]
	\endgroup
	where $\hat{p}^\text{KDE}_\text{data}$ denotes the \textit{Kernel Density Estimate} (KDE) \cite{silverman1986density} of the regression label distribution, and $\sigma_\text{KDE}^2$ is the kernel bandwidth parameter.

    Based on Eq.~\eqref{eq:joint_estimate_w_kde_vicinal}, the denoising loss for $D_{\bm{\theta}}$ under a fixed covariance matrix $\bm{\Sigma}$ is defined as
	\begingroup
	\[
		&\hat{\mathcal{L}}(\bm{\theta};\bm{\Sigma}) \\
		= & \mathbb{E}_{\substack{(\bm{X},Y)\sim \hat{p}^v_\text{data}(\bm{x},y) \\ \bm{\varepsilon}\sim\mathcal{N}(\bm{0},\bm{\Sigma})}}\left[ \left\| \Lambda_{\bm{\Sigma}}^{\frac{1}{2}} \left( D_{\bm{\theta}}(\bm{X}+\bm{\varepsilon},y,\bm{\Sigma}) - \bm{X} \right) \right\|_2^2 \right] \\
		= & \frac{C}{N}\sum_{i=1}^N\sum_{j=1}^N \mathbb{E}_{\substack{ \bm{\varepsilon}\sim\mathcal{N}(\bm{0},\bm{\Sigma}) \\ \eta\sim\mathcal{N}(0,\sigma_\text{KDE}^2) }} W_{i,y_j+\eta}\left\| \Lambda_{\bm{\Sigma}}^{\frac{1}{2}} \right. \\
		& \qquad\qquad\qquad  \left. \vphantom{\Lambda_{\bm{\Sigma}}^{\frac{1}{2}}} \cdot \left( D_{\bm{\theta}}(\bm{x}_i+\bm{\varepsilon},y_j+\eta,\bm{\Sigma}) - \bm{x}_i \right) \right\|_2^2
		\label{eq:hat_L_theta_Sigma}
	\]
	\endgroup
	where $C$ absorbs all normalization constants, $\eta\triangleq y-y_j$, and the noise-weighting matrix is defined as
    \[
		\Lambda_{\bm{\Sigma}} = \text{Diag}\left( \left[ \frac{\Sigma_{ii}+\sigma_\text{data}^2}{\sigma_\text{data}^2\Sigma_{ii}} \right]_{i=1}^d \right),
          \label{eq:noise_weighting}
	\]
	where $\Sigma_{ii}$ denotes the $i$-th diagonal element of $\bm{\Sigma}$. 

    Finally, the overall training objective is obtained by taking the expectation of $\hat{\mathcal{L}}(\bm{\theta};\bm{\Sigma})$ over the distribution of noise levels:
	\[
		\hat{\mathcal{L}}(\bm{\theta}) = \mathbb{E}_{\log(\sigma) \sim \mathcal{N}(P_\text{mean}, P_\text{std}^2)}\hat{\mathcal{L}}(\bm{\theta};\bm{\Sigma}),
	\]
	where $\bm{\Sigma}$ is computed using Eq.~\eqref{eq:Sigma} with $\sigma(t)$ replaced by $\sigma$.
    
	\subsubsection{Network Design}\label{sec:network_design}
    
    In the EDM framework, Karras \emph{et al.}~\cite{karras2022elucidating} adopt the ADM-style UNet architecture~\cite{dhariwal2021diffusion}, with minor modifications, as the backbone for $F_{\bm{\theta}}$ in Eq.~\eqref{eq:preconditioning}. However, as shown in Table~\ref{tab:ab_effect_network}, the EDM UNet does not outperform the CCDM UNet~\cite{ding2025ccdm} on CCGM tasks. Therefore, we use the CCDM UNet as the backbone for iCCDM in most experiments. For the $256\times 256$ experiments, we propose adopting the vision transformer-based DiT backbone~\cite{peebles2023scalable} due to its superior performance.
    

    In addition, we replace the five-layer fully connected covariance embedding network $\phi(y)$ used in CCDM~\cite{ding2025ccdm} with a five-layer convolutional neural network. This modification substantially reduces the number of model parameters without sacrificing generation quality, as demonstrated in Table~\ref{tab:ab_effect_y2cov_network}.

	\subsection{Second-Order Sampling Methods in Matrix Form}\label{sec:sampling}

    For sample generation, we propose a deterministic sampling method based on Heun’s second-order scheme, as well as a stochastic sampling approach. These two methods are detailed in Algorithms~\ref{alg:deterministic_sampling} and \ref{alg:stochastic_sampling}, respectively, corresponding to an ODE solver and an SDE solver. Both algorithms inherit several hyperparameters from the original EDM sampling procedures~\cite{karras2022elucidating}, including $\sigma_\text{max}$, $S_\text{noise}$, $S_\text{churn}$, $S_\text{tmin}$, and $S_\text{tmax}$, which are tuned following the guidelines in~\cite{karras2022elucidating}.

    For both sampling methods, we adopt classifier-free guidance (CFG)~\cite{ho2021classifier} to incorporate the regression label 
    $y$ into the generation process. Specifically, each output of $D_{\bm{\theta}}$ in Algorithms~\ref{alg:deterministic_sampling} and \ref{alg:stochastic_sampling} is computed as a linear combination of its conditional and unconditional outputs, with a guidance scale $\gamma\in[1.2,2]$ to balance image diversity and label consistency. We refer the reader to Section~D of CCDM~\cite{ding2025ccdm} for additional details.
    
    In all experiments, the number of sampling steps for both solvers is fixed to 32. As shown in Table~\ref{tab:ab_effect_sampler}, increasing the number of sampling steps does not further improve sampling quality, which contrasts with the findings reported in EDM~\cite{karras2022elucidating}. Moreover, Table~\ref{tab:ab_effect_sampler} indicates that the SDE solver often outperforms the ODE solver---again differing from the conclusions drawn in the standard EDM setting.
     	
	\begin{algorithm}[!h]
		\caption{Deterministic conditional sampling using Heun's $2^\text{nd}$ order method.}
		\label{alg:deterministic_sampling} 
		\begin{algorithmic}[1]
			\Require Condition $y$, weight $\lambda_y$, denoiser $D_{\bm{\theta}}$, number of ODE solver iterations $N$, and corresponding time steps $0=t_0<t_1<\cdots<t_N=\sigma_\text{max}$
			\State Sample $\bm{x}_N\sim \mathcal{N}(\bm{0}, \bm{\Sigma}(t_N,y;\lambda_y))$ 
			\State $i\leftarrow N$
			\While{$i>0$}
				\State $\bm{r}_i \leftarrow (\bm{x}_i - D_{\bm{\theta}}(\bm{x}_i,y,\bm{\Sigma}(t_i,y;\lambda_y)))$
				
				\State $\bm{d}_i\leftarrow \frac{1}{2}\dot{\bm{\Sigma}}(t_i,y;\lambda_y)\bm{\Sigma}(t_i,y;\lambda_y)^{-1} \bm{r}_i $
		
				\State $\bm{x}_{i-1}\leftarrow \bm{x}_i + (t_{i-1}-t_i)\bm{d}_i$
				\If{$t_{i-1}\neq 0$}
					\State $\bm{r}_{i-1} \leftarrow (\bm{x}_{i-1} - D_{\bm{\theta}}(\bm{x}_{i-1},y,\bm{\Sigma}(t_{i-1},y;\lambda_y)))$
				
					\State $\bm{d}_i^\prime\leftarrow \frac{1}{2}\dot{\bm{\Sigma}}(t_{i-1},y;\lambda_y)\bm{\Sigma}(t_{i-1},y;\lambda_y)^{-1}\bm{r}_{i-1}$
				
					\State $\bm{x}_{i-1}\leftarrow \bm{x}_i + (t_{i-1}-t_i)(\frac{1}{2}\bm{d}_i+\frac{1}{2}\bm{d}_i^\prime)$
				\EndIf
				\State $i\leftarrow (i-1)$
			\EndWhile
			\State \Return $\bm{x}_0$
		\end{algorithmic}
	\end{algorithm}
	
	\begin{algorithm}[!h]
		\caption{EDM-style Stochastic Conditional Sampling.}
		\label{alg:stochastic_sampling} 
		\begin{algorithmic}[1]
			\Require Condition $y$, weight $\lambda_y$, denoiser $D_{\bm{\theta}}$, number of SDE solver iterations $N$, corresponding time steps $0=t_0<t_1<\cdots<t_N=\sigma_\text{max}$, and four hyperparameters including $S_\text{noise}$, $S_\text{churn}$, $S_\text{tmin}$, and $S_\text{tmax}$.
			\State Sample $\bm{x}_N\sim \mathcal{N}(\bm{0}, \bm{\Sigma}(t_N,y;\lambda_y))$ 
			\State $i\leftarrow N$
			\While{$i>0$}
				\State Sample $\bm{\varepsilon}_i\sim\mathcal{N}(\bm{0},S_\text{noise}^2\bm{I})$
				\State Add noise to reach a higher noise level $\hat{t}_i\leftarrow t_i+\gamma_i t_i$, where $\gamma_i = \min\left( \frac{S_\text{churn}}{N},\sqrt{2}-1 \right) \cdot\mathds{1}_{\{t_i\in[S_\text{tmin}, S_\text{tmax}]\}}$.
    
				
				\State $\hat{\bm{x}}_i\leftarrow \bm{x}_i + \left( \bm{\Sigma}(\hat{t}_i,y;\lambda_y) - \bm{\Sigma}(t_i,y;\lambda_y) \right)^{\frac{1}{2}}\bm{\varepsilon}_i$
				
				\State $\bm{r}_i \leftarrow (\hat{\bm{x}}_i - D_{\bm{\theta}}(\hat{\bm{x}}_i,y,\bm{\Sigma}(\hat{t}_i,y;\lambda_y)))$
				
				\State $\bm{d}_i\leftarrow \frac{1}{2}\dot{\bm{\Sigma}}(\hat{t}_i,y;\lambda_y)\bm{\Sigma}(\hat{t}_i,y;\lambda_y)^{-1}\bm{r}_i$
				
				\State $\bm{x}_{i-1}\leftarrow \hat{\bm{x}}_i + (t_{i-1}-\hat{t}_i)\bm{d}_i $
				
				\If{$t_{i-1}\neq 0$}
					\State $\bm{r}_{i-1} \leftarrow (\bm{x}_{i-1} - D_{\bm{\theta}}(\bm{x}_{i-1},y,\bm{\Sigma}(t_{i-1},y;\lambda_y)))$
					
					\State $\bm{d}_i^\prime\leftarrow \frac{1}{2}\dot{\bm{\Sigma}}(t_{i-1},y;\lambda_y)\bm{\Sigma}(t_{i-1},y;\lambda_y)^{-1}\bm{r}_{i-1}$
					
					\State $\bm{x}_{i-1}\leftarrow \hat{\bm{x}}_i + (t_{i-1}-\hat{t}_i)(\frac{1}{2}\bm{d}_i+\frac{1}{2}\bm{d}_i^\prime)$
				\EndIf

				\State $i\leftarrow i-1$
			\EndWhile
			\State \Return $\bm{x}_0$
		\end{algorithmic}
	\end{algorithm}

	\section{Experiment} \label{sec:experiment}

	\begin{table*}[!h]  
		\setlength{\tabcolsep}{1mm} 
		\footnotesize
			\centering
			\caption{ Average Quality of Low-Resolution Generated Images ($\leq 128\times128$).
            Arrows ($\downarrow$ or $\uparrow$) indicate whether lower or higher values are preferred. The best and second-best results according to the primary metric (SFID) are highlighted in \textbf{bold} and \underline{underlined}, respectively. Evaluation results corresponding to obvious failure cases (e.g., severe label inconsistency and mode collapse) are highlighted in \fail{red}.  Cell-200 does not contain categorical labels; therefore, its Diversity metric is not available. }
			\begin{tabular}{c|c|c|c|cccc}
				\toprule
				\begin{tabular}[c]{@{}c@{}} \textbf{Dataset} \\ (resolution)\end{tabular} & \begin{tabular}[c]{@{}c@{}} \textbf{Condition} \\ \textbf{Type} \end{tabular} & \begin{tabular}[c]{@{}c@{}} \textbf{Framework} \\ \textbf{Type} \end{tabular} & \textbf{Method} & \begin{tabular}[c]{@{}c@{}} \textbf{SFID} \\ (primary)\end{tabular} $\downarrow$   & \textbf{NIQE} $\downarrow$ & \textbf{Diversity} $\uparrow$ & \begin{tabular}[c]{@{}c@{}} \textbf{Label} \\ \textbf{Score} \end{tabular} $\downarrow$ \\
				\midrule
				
				\multirow{10}[2]{*}{\begin{tabular}[c]{@{}c@{}} \textbf{RC-49} \\ {(64$\times$64)}\end{tabular}} & \multirow{4}[0]{*}{Text} & \multirow{4}[0]{*}{Diffusion} & SD v1.5 (CVPR'22)~\cite{rombach2022high} & \fail{0.546}  & 2.910  & \fail{1.730}  & \fail{10.743}  \\
				&   &    & SD 3 Medium (ICML'24)~\cite{esser2024scaling} & \fail{1.215}  & 1.996  & \fail{2.272}  & \fail{28.726} \\
				&   &    & FLUX.1 (arXiv'25)~\cite{labs2025flux1kontextflowmatching} & \fail{0.515}  & 2.330  & \fail{2.512}  & \fail{12.751}  \\
                &   &    & Qwen-Image (arXiv'25)~\cite{wu2025qwen} & \fail{1.132} & 2.436 & \fail{2.470} & \fail{27.757}  \\
				\cdashline{2-8} 
				\specialrule{0em}{1pt}{1pt}
				
				& \multirow{6}[0]{*}{Continuous} & \multirow{3}[0]{*}{GAN} & CcGAN (T-PAMI'23)~\cite{ding2023ccgan} & 0.126  & 1.809  & 3.451  & 2.655  \\
				&   &    & Dual-NDA (AAAI'24)~\cite{ding2024turning} & 0.148  & 1.808  & 3.344  & 2.211  \\
				&   &    & CcGAN-AVAR (arXiv'26)~\cite{ding2025imbalance} & \textbf{0.042}   & {1.734} & {3.723}   & 1.270   \\
				\cdashline{3-8}
				\specialrule{0em}{1pt}{1pt}
				& & \multirow{3}[0]{*}{Diffusion}  & CcDPM (AAAI'24)~\cite{zhao2024ccdpm} & \fail{0.970}  & 2.153  & 3.581  & \fail{24.174}  \\
				&   &    & CCDM (TMM'26)~\cite{ding2025ccdm} & 0.049  & 2.086 & {3.698}  & {1.074}  \\
				&   &    & \textbf{iCCDM (ours)} & \underline{0.049}  & {1.718}  & 3.696  & {0.813}  \\
				\hline\specialrule{0em}{1pt}{1pt}
				
				\multirow{10}[2]{*}{\begin{tabular}[c]{@{}c@{}} \textbf{UTKFace} \\ {(64$\times$64)}\end{tabular}} & \multirow{4}[0]{*}{Text} & \multirow{4}[0]{*}{Diffusion} & SD v1.5 (CVPR'22)~\cite{rombach2022high} & 1.013  & 2.607  & 1.091  & {5.027}  \\
				&   &    & SD 3 Medium (ICML'24)~\cite{esser2024scaling} & \fail{2.012}  & 2.068 & \fail{0.426}  & {5.452}  \\
				&   &    & FLUX.1 (arXiv'25)~\cite{labs2025flux1kontextflowmatching} & 0.816 & 1.915 & \fail{0.713}  & 9.174  \\
                &   &    & Qwen-Image (arXiv'25)~\cite{wu2025qwen} & 1.973   & 2.135   & 1.051   & \fail{10.770}   \\
				\cdashline{2-8} 
				\specialrule{0em}{1pt}{1pt}
				
				& \multirow{6}[0]{*}{Continuous} & \multirow{3}[0]{*}{GAN} & CcGAN (T-PAMI'23)~\cite{ding2023ccgan} & 0.413 & 1.733 & {1.329} & 8.240 \\
				&   &    & Dual-NDA (AAAI'24)~\cite{ding2024turning} & 0.396 & 1.678 & {1.298} & 6.765 \\
				&   &    & CcGAN-AVAR (arXiv'26)~\cite{ding2025imbalance} & \underline{0.356} & 1.691 & 1.278 & 6.696  \\
				\cdashline{3-8}
				\specialrule{0em}{1pt}{1pt}
				& & \multirow{3}[0]{*}{Diffusion}  & CcDPM (AAAI'24)~\cite{zhao2024ccdpm} & 0.466 & {1.560} & 1.211 & 6.868 \\
				&   &    & CCDM (TMM'26)~\cite{ding2025ccdm} & 0.363 & {1.542} & 1.184 & {6.164} \\
				&   &    & \textbf{iCCDM (ours)} & \textbf{0.353}  & {1.561}  & 1.146  & {5.788} \\
				\hline\specialrule{0em}{1pt}{1pt}
				
				\multirow{10}[2]{*}{\begin{tabular}[c]{@{}c@{}} \textbf{Steering Angle} \\ {(64$\times$64)}\end{tabular}} & \multirow{4}[0]{*}{Text} & \multirow{4}[0]{*}{Diffusion} & SD v1.5 (CVPR'22)~\cite{rombach2022high} & 1.975  & 1.676  & 1.078  & 8.907  \\
				&   &    & SD 3 Medium (ICML'24)~\cite{esser2024scaling} & \fail{4.007}  & 1.685 & 0.833  & \fail{19.885}  \\
				&   &    & FLUX.1 (arXiv'25)~\cite{labs2025flux1kontextflowmatching} & 1.795  & 1.766 & 1.416  & \fail{19.911}  \\
                &   &    & Qwen-Image (arXiv'25)~\cite{wu2025qwen} & 2.557 & 2.259 & 1.357   &  \fail{20.734} \\
				\cdashline{2-8} 
				\specialrule{0em}{1pt}{1pt}
				
				& \multirow{6}[0]{*}{Continuous} & \multirow{3}[0]{*}{GAN} & CcGAN (T-PAMI'23)~\cite{ding2023ccgan} & 1.334 & 1.784 & {1.234} & 14.807 \\
				&   &    & Dual-NDA (AAAI'24)~\cite{ding2024turning} & 1.114 & {1.738} & {1.251} & 11.809 \\
				&   &    & CcGAN-AVAR (arXiv'26)~\cite{ding2025imbalance} & 0.809 & 1.800  & 1.204 & 6.963  \\
				\cdashline{3-8}
				\specialrule{0em}{1pt}{1pt}
				& & \multirow{3}[0]{*}{Diffusion}  & CcDPM (AAAI'24)~\cite{zhao2024ccdpm} & 0.939  & 1.761 & 1.150  &  10.999 \\
				&   &    & CCDM (TMM'26)~\cite{ding2025ccdm} & \underline{0.742} & 1.778 & 1.088 & {5.823} \\
				&   &    & \textbf{iCCDM (ours)} & \textbf{0.501} & {1.700} & 1.093  & {4.891}  \\
				\hline\specialrule{0em}{1pt}{1pt}
				
				\multirow{6}[2]{*}{\begin{tabular}[c]{@{}c@{}} \textbf{Cell-200} \\ {(64$\times$64)}\end{tabular}} & \multirow{6}[0]{*}{Continuous} & \multirow{3}[0]{*}{GAN} & CcGAN (T-PAMI'23)~\cite{ding2023ccgan} & 6.848  & 1.298  & ---     & 5.210 \\
				&   &    & Dual-NDA (AAAI'24)~\cite{ding2024turning} & 6.439 & {1.095} & ---   & 6.420 \\
				&   &    & CcGAN-AVAR (arXiv'26)~\cite{ding2025imbalance} & 7.665 & 1.610 & --- & 7.499 \\
				\cdashline{3-8}
				\specialrule{0em}{1pt}{1pt}
				& & \multirow{3}[0]{*}{Diffusion}  & CcDPM (AAAI'24)~\cite{zhao2024ccdpm} & 6.394  & 1.208  & --- & 3.196  \\
				&   &    & CCDM (TMM'26)~\cite{ding2025ccdm} & \textbf{5.122}  & 1.184 & --- & {2.941} \\
				&   &    & \textbf{iCCDM (ours)} & \underline{5.276} & {0.961}  & --- & {2.983}  \\
				\hline\specialrule{0em}{1pt}{1pt}
				
				\multirow{10}[2]{*}{\begin{tabular}[c]{@{}c@{}} \textbf{UTKFace} \\ {(128$\times$128)}\end{tabular}} & \multirow{4}[0]{*}{Text} & \multirow{4}[0]{*}{Diffusion} & SD v1.5 (CVPR'22)~\cite{rombach2022high} & 0.768  & 2.412  & 1.059  & {5.195}  \\
				&   &    & SD 3 Medium (ICML'24)~\cite{esser2024scaling} & \fail{1.582} & 1.901 & \fail{0.396} & {5.019}  \\
				&   &    & FLUX.1 (arXiv'25)~\cite{labs2025flux1kontextflowmatching} & 0.643  & 1.748  & \fail{0.661}  & 9.281  \\
                &   &    & Qwen-Image (arXiv'25)~\cite{wu2025qwen} & 1.517   & 1.653   & 0.971   & 9.731   \\
				\cdashline{2-8} 
				\specialrule{0em}{1pt}{1pt}
				
				& \multirow{6}[0]{*}{Continuous} & \multirow{3}[0]{*}{GAN} & CcGAN (T-PAMI'23)~\cite{ding2023ccgan} & 0.367 & 1.113 & 1.199 & 7.747 \\
				&   &    & Dual-NDA (AAAI'24)~\cite{ding2024turning} & 0.361 & 1.081 & {1.257} & 6.310 \\
				&   &    & CcGAN-AVAR (arXiv'26)~\cite{ding2025imbalance} & \underline{0.297} & 1.173 & {1.251} & 6.586  \\
				\cdashline{3-8}
				\specialrule{0em}{1pt}{1pt}
				& & \multirow{3}[0]{*}{Diffusion}  & CcDPM (AAAI'24)~\cite{zhao2024ccdpm} & 0.529 & 1.114 & 1.195  & 7.933 \\
				&   &    & CCDM (TMM'26)~\cite{ding2025ccdm} & 0.319 & {1.077} & 1.178 & {6.359} \\
				&   &    & \textbf{iCCDM (ours)} & \textbf{0.294} & {1.048}  & 1.225  & {5.823}  \\
				\hline\specialrule{0em}{1pt}{1pt}
				
				\multirow{10}[2]{*}{\begin{tabular}[c]{@{}c@{}} \textbf{Steering Angle} \\ {(128$\times$128)}\end{tabular}} & \multirow{4}[0]{*}{Text} & \multirow{4}[0]{*}{Diffusion} & SD v1.5 (CVPR'22)~\cite{rombach2022high} & 1.843  & 3.543  & 1.174  & 9.479  \\
				&   &    & SD 3 Medium (ICML'24)~\cite{esser2024scaling} & \fail{4.191} & 4.071  & \fail{0.512}  & \fail{20.072}  \\
				&   &    & FLUX.1 (arXiv'25)~\cite{labs2025flux1kontextflowmatching} & 1.676 & 3.090  & 1.072  & \fail{20.526} \\
                &   &    & Qwen-Image (arXiv'25)~\cite{wu2025qwen} & 2.910 & 3.749 & 1.062   & \fail{18.455}   \\
				\cdashline{2-8} 
				\specialrule{0em}{1pt}{1pt}
				
				& \multirow{6}[0]{*}{Continuous} & \multirow{3}[0]{*}{GAN} & CcGAN (T-PAMI'23)~\cite{ding2023ccgan} & {1.689} & {2.411} & {1.088} & \fail{18.438} \\
				&   &    & Dual-NDA (AAAI'24)~\cite{ding2024turning} & {1.390} & {2.135} & {1.133} & {14.099} \\
				&   &    & CcGAN-AVAR (arXiv'26)~\cite{ding2025imbalance} & \underline{0.888} & 2.288 & 1.123 & {7.507} \\
				\cdashline{3-8}
				\specialrule{0em}{1pt}{1pt}
				& & \multirow{3}[0]{*}{Diffusion}  & CcDPM (AAAI'24)~\cite{zhao2024ccdpm} & 1.285  & 1.989  & {1.203} & \fail{18.325} \\
				&   &    & CCDM (TMM'26)~\cite{ding2025ccdm} & 0.987 & {1.977} & 1.118 & 11.829 \\
				&   &    & \textbf{iCCDM (ours)} & \textbf{0.742} & {1.973} & 1.007  & {5.652}  \\
				\bottomrule
			\end{tabular}%
			\label{tab:main_results_low}%
	\end{table*}%

	\begin{table*}[!ht]  
		\setlength{\tabcolsep}{1mm} 
		\footnotesize
		\centering
		\caption{ Average Quality of High-Resolution Generated Images ($\geq 192\times 192$).
            Arrows ($\downarrow$ or $\uparrow$) indicate whether lower or higher values are preferred. The best and second-best results according to the primary metric (SFID) are highlighted in \textbf{bold} and \underline{underlined}, respectively. Evaluation results corresponding to obvious failure cases (e.g., severe label inconsistency and mode collapse) are highlighted in \fail{red}. }
		\begin{tabular}{c|c|c|c|cccc}
			\toprule
			\begin{tabular}[c]{@{}c@{}} \textbf{Dataset} \\ (resolution)\end{tabular} & \begin{tabular}[c]{@{}c@{}} \textbf{Condition} \\ \textbf{Type} \end{tabular} & \begin{tabular}[c]{@{}c@{}} \textbf{Framework} \\ \textbf{Type} \end{tabular} & \textbf{Method} & \begin{tabular}[c]{@{}c@{}} \textbf{SFID} \\ (primary)\end{tabular} $\downarrow$  & \textbf{NIQE} $\downarrow$ & \textbf{Diversity} $\uparrow$ & \begin{tabular}[c]{@{}c@{}} \textbf{Label} \\ \textbf{Score} \end{tabular} $\downarrow$ \\
			\midrule
			
			\multirow{10}[2]{*}{\begin{tabular}[c]{@{}c@{}} \textbf{UTKFace} \\ {(192$\times$192)}\end{tabular}} & \multirow{4}[0]{*}{Text} & \multirow{4}[0]{*}{Diffusion} & SD v1.5 (CVPR'22)~\cite{rombach2022high} & 0.928  & 2.941  & 1.053  & {5.065}  \\
			&   &    & SD 3 Medium (ICML'24)~\cite{esser2024scaling} & \fail{1.682} & 2.436 & \fail{0.383} & {5.188}  \\
			&   &    & FLUX.1 (arXiv'25)~\cite{labs2025flux1kontextflowmatching} & 0.771 & 2.499  & 0.602 & 9.191 \\
            &   &    & Qwen-Image (arXiv'25)~\cite{wu2025qwen} & 1.617   & 1.909   & 0.943   & 9.868   \\
			\cdashline{2-8} 
			\specialrule{0em}{1pt}{1pt}
			
			& \multirow{6}[0]{*}{Continuous} & \multirow{3}[0]{*}{GAN} & CcGAN (T-PAMI'23)~\cite{ding2023ccgan} & 0.499 & 1.661 & {1.207} & 7.885 \\
			&   &    & Dual-NDA (AAAI'24)~\cite{ding2024turning} & 0.487 & {1.483} & 1.201 & 6.730 \\
			&   &    & CcGAN-AVAR (arXiv'26)~\cite{ding2025imbalance} & \underline{0.435} & 1.588 & 1.187 & 6.377  \\
			\cdashline{3-8}
			\specialrule{0em}{1pt}{1pt}
			& & \multirow{3}[0]{*}{Diffusion}  & CcDPM (AAAI'24)~\cite{zhao2024ccdpm} & 0.970  &  1.522  &  1.187  & \fail{11.224} \\
			&   &    & CCDM (TMM'26)~\cite{ding2025ccdm} & 0.467 &  {1.242} &  1.148 & 7.336 \\
			&   &    & \textbf{iCCDM (ours)} & \textbf{0.433} & 1.850 & {1.210} & 6.323 \\
			\hline\specialrule{0em}{1pt}{1pt}
			
			\multirow{10}[2]{*}{\begin{tabular}[c]{@{}c@{}} \textbf{UTKFace} \\ {(256$\times$256)}\end{tabular}} & \multirow{4}[0]{*}{Text} & \multirow{4}[0]{*}{Diffusion} & SD v1.5 (CVPR'22)~\cite{rombach2022high} & 0.699 & 2.308  & 0.977  & 8.225  \\
			&   &    & SD 3 Medium (ICML'24)~\cite{esser2024scaling} & \fail{0.906}  & 1.524  & \fail{0.454}  & 5.850  \\
			&   &    & FLUX.1 (arXiv'25)~\cite{labs2025flux1kontextflowmatching} & 0.369  & 1.576  & \fail{0.614}  & 9.010  \\
            &   &    & Qwen-Image (arXiv'25)~\cite{wu2025qwen} & 0.953   & 1.362   & 0.980   & \fail{10.941}   \\
			\cdashline{2-8} 
			\specialrule{0em}{1pt}{1pt}
			
			& \multirow{6}[0]{*}{Continuous} & \multirow{3}[0]{*}{GAN} & CcGAN (T-PAMI'23)~\cite{ding2023ccgan} & 0.347 & {1.308} & 1.053 & 6.010 \\
			&   &    & Dual-NDA (AAAI'24)~\cite{ding2024turning} & 0.319 & 1.381 & {1.216} & {5.712} \\
			&   &    & CcGAN-AVAR (arXiv'26)~\cite{ding2025imbalance} & \textbf{0.196}  & 1.347 & {1.219} & {5.798} \\
			\cdashline{3-8}
			\specialrule{0em}{1pt}{1pt}
			& & \multirow{3}[0]{*}{Diffusion}  & CcDPM (AAAI'24)~\cite{zhao2024ccdpm} & 0.559 & 1.552  & 0.854 & \fail{12.121} \\
			&   &    & CCDM (TMM'26)~\cite{ding2025ccdm} & 0.268 & 1.364 & 1.121 & 7.555 \\
			&   &    & \textbf{iCCDM (ours)} & \underline{0.218} & {1.207} & 1.142 & 6.199 \\
			\hline\specialrule{0em}{1pt}{1pt}
			
			\multirow{10}[2]{*}{\begin{tabular}[c]{@{}c@{}} \textbf{Steering Angle} \\ {(256$\times$256)}\end{tabular}} & \multirow{4}[0]{*}{Text} & \multirow{4}[0]{*}{Diffusion} & SD v1.5 (CVPR'22)~\cite{rombach2022high} & 1.023  & 1.747  & 0.937  & 5.458  \\
			&   &    & SD 3 Medium (ICML'24)~\cite{esser2024scaling} & \fail{1.859}  & 1.595  & 0.907  & \fail{19.806}  \\
			&   &    & FLUX.1 (arXiv'25)~\cite{labs2025flux1kontextflowmatching} & 0.945  & 2.353  & 1.078  & \fail{17.228}  \\
            &   &    & Qwen-Image (arXiv'25)~\cite{wu2025qwen} & 1.383   & 2.525   & 1.001   & \fail{16.034} \\
			\cdashline{2-8} 
			\specialrule{0em}{1pt}{1pt}
			
			& \multirow{6}[0]{*}{Continuous} & \multirow{3}[0]{*}{GAN} & CcGAN (T-PAMI'23)~\cite{ding2023ccgan} & 0.984 & 1.999 & 1.054 & 8.399 \\
			&   &    & Dual-NDA (AAAI'24)~\cite{ding2024turning} & 0.967 & 1.834 & 0.951 & 8.338 \\
			&   &    & CcGAN-AVAR (arXiv'26)~\cite{ding2025imbalance} & \underline{0.683}  & 1.966  & 1.113  & {4.958} \\
			\cdashline{3-8}
			\specialrule{0em}{1pt}{1pt}
			& & \multirow{3}[0]{*}{Diffusion}  & CcDPM (AAAI'24)~\cite{zhao2024ccdpm} & 1.217 & 1.748 & {1.330} & \fail{29.239} \\
			&   &    & CCDM (TMM'26)~\cite{ding2025ccdm} & 0.902 & {1.616} & {1.151} & \fail{23.057} \\
			&   &    & \textbf{iCCDM (ours)} & \textbf{0.582}  & {1.390}  & 0.981 & {4.743}  \\
			
			\bottomrule
		\end{tabular}%
		\label{tab:main_results_high}%
	\end{table*}%

	\begin{figure}[!htbp] 
		\centering
		\includegraphics[width=0.8\linewidth]{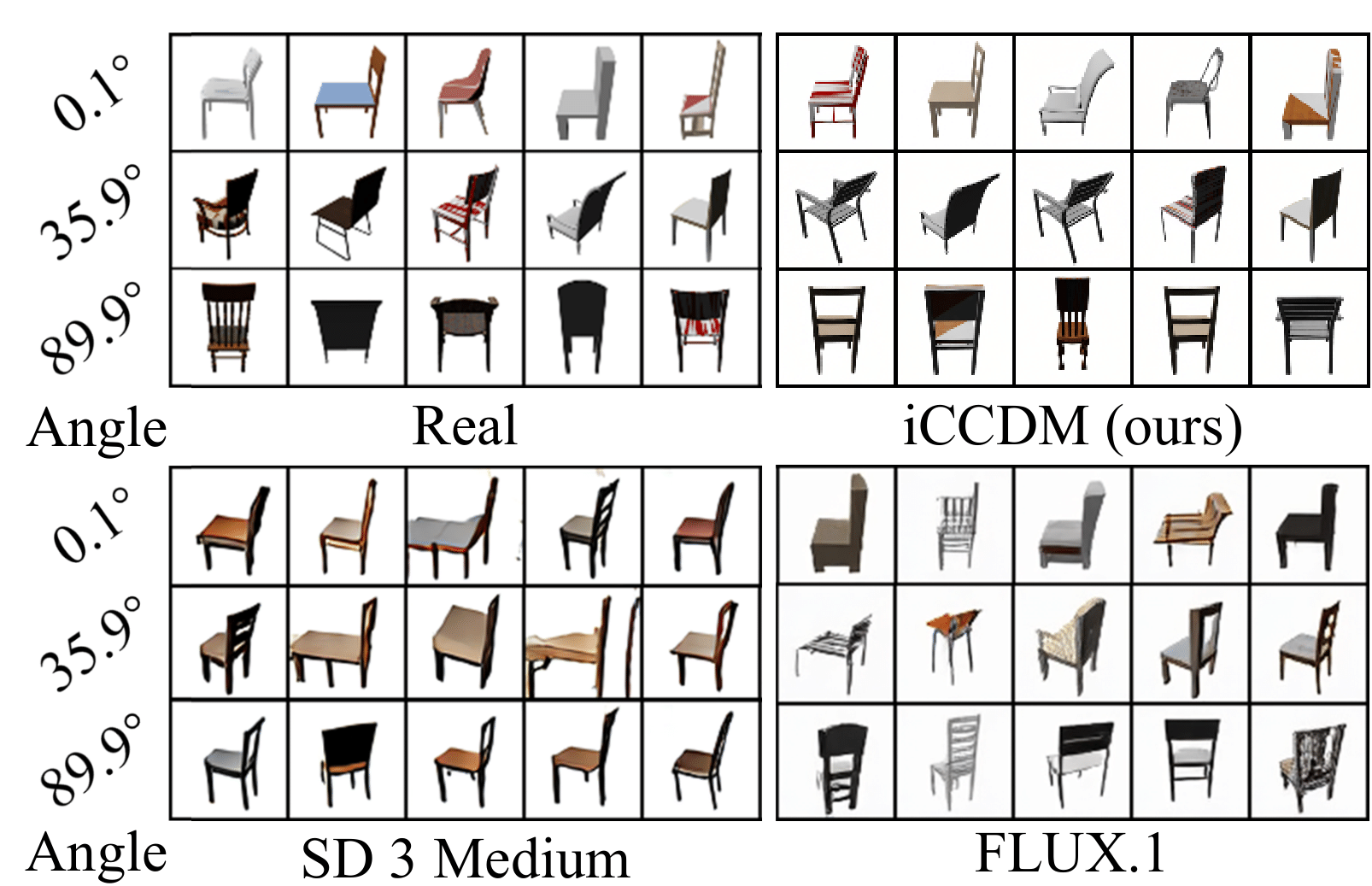}
		\caption{ An Illustrative Example Highlighting the Superiority of iCCDM over State-of-the-Art Text-to-Image Models for Rotation-Angle Conditioning on RC-49. }
		\label{fig:example_imgs_real_and_fake}
	\end{figure}

    \begin{table}[!htbp]
		\setlength{\tabcolsep}{1mm} 
		\centering
		\caption{ Comparison of Sampling Quality and Efficiency Across Diffusion-Based Methods on the Steering Angle Dataset ($64\times 64$). Sampler names are given in parentheses, with the corresponding numbers of sampling steps indicated after hyphens.  }
			\begin{tabular}{c|c|cc}
				\toprule
				\textbf{Method} & \textbf{SFID}  & \begin{tabular}[c]{@{}c@{}} {\textbf{Speed}} \\ (img/sec) \end{tabular} & \begin{tabular}[c]{@{}c@{}} {\textbf{Memory}} \\ (GiB) \end{tabular}  \\
				\midrule
				SD v1.5 (DDIM-28) & 1.975 & 52.63 & 2.50 \\
				SD 3 Medium (ODE-28) & 4.010 & 2.88 & 19.74 \\
				FLUX.1 (ODE-14) & 1.795 & 0.46 & 42.08 \\
                Qwen-Image (ODE-14) & 2.557 & 3.23 & 56.9 \\
				\cdashline{1-4}\specialrule{0em}{1pt}{1pt}
				CcDPM (DDIM-150) & 0.939 & 3.33 & 2.28 \\
                CCDM (DDIM-150) & 0.742 & 3.33 & 2.74 \\
 				\cdashline{1-4  }\specialrule{0em}{1pt}{1pt}
				iCCDM (SDE-32) & 0.501 & 12.82 & 2.50 \\
				\bottomrule
			\end{tabular}%
		\label{tab:ab_sampling_cost}%
	\end{table}%

	\begin{table}[!h]  
		\setlength{\tabcolsep}{1mm}
		\centering
		\caption{ Component-wise Analysis of iCCDM. HFV and HAV denote hard fixed and hard adaptive vicinities, respectively. When $\lambda_y = 0$ and AV are disabled, iCCDM reduces to the vanilla EDM.}
			\begin{tabular}{c|ccc|cccc}
				\toprule
				\multirow{2}[0]{*}{\textbf{Dataset}}  & \multicolumn{3}{c|}{\textbf{Configuration}} & \multirow{2}[0]{*}{\begin{tabular}[c]{@{}c@{}} \textbf{SFID} \\ {(primary)} \end{tabular}} & \multirow{2}[0]{*}{\textbf{NIQE}} & \multirow{2}[0]{*}{\begin{tabular}[c]{@{}c@{}} \textbf{Diver} \\ \textbf{-sity} \end{tabular}} & \multirow{2}[0]{*}{\begin{tabular}[c]{@{}c@{}} \textbf{Label} \\ \textbf{Score} \end{tabular}} \\
				\textbf{} &  \textbf{HFV} &  \textbf{HAV} & $\lambda_y>0$ & \textbf{} & \textbf{} & \textbf{} & \textbf{} \\
				\midrule
				\multirow{5}[0]{*}{\begin{tabular}[c]{@{}c@{}} \textbf{UTKFace} \\ ($64\times 64$) \end{tabular}} &    &    &   & 0.356 & 1.730  & 1.108  & 6.001 \\
				\cdashline{2-8} \specialrule{0em}{1pt}{1pt}
				&\checkmark &  &  & 0.353 & 1.716  & 1.157  & 5.791 \\
				&\checkmark &  & \checkmark & 0.353 & 1.593  & 1.151  & 5.710  \\
				\cdashline{2-8} \specialrule{0em}{1pt}{1pt}
				& &  \checkmark &  & 0.356 & 1.696  & 1.156  & 5.767  \\
				& & \checkmark & \checkmark & 0.353  & 1.561  & 1.146  & 5.788 \\
				\cline{1-8} \specialrule{0em}{1pt}{1pt}
				\multirow{5}[0]{*}{\begin{tabular}[c]{@{}c@{}} \textbf{Steering} \\ \textbf{Angle} \\ ($64\times 64$) \end{tabular}} &    &    &   & 0.813 & 1.750  & 1.083  & 5.044  \\
				\cdashline{2-8} \specialrule{0em}{1pt}{1pt}
				& \checkmark &  &  & 0.644  & 1.764  & 1.063  & 4.333  \\
				& \checkmark &  & \checkmark & 0.561  & 1.731  & 1.123 & 5.136  \\
				\cdashline{2-8} \specialrule{0em}{1pt}{1pt}
				& &  \checkmark &  & 0.546 & 1.719 & 1.119 & 5.010  \\
				& & \checkmark & \checkmark & 0.501 & 1.700 & 1.093  & 4.891 \\
				\bottomrule
			\end{tabular}%
		\label{tab:ablation_components}%
	\end{table}%

    \begin{figure}[!h]  
		\centering
		\includegraphics[width=1\linewidth]{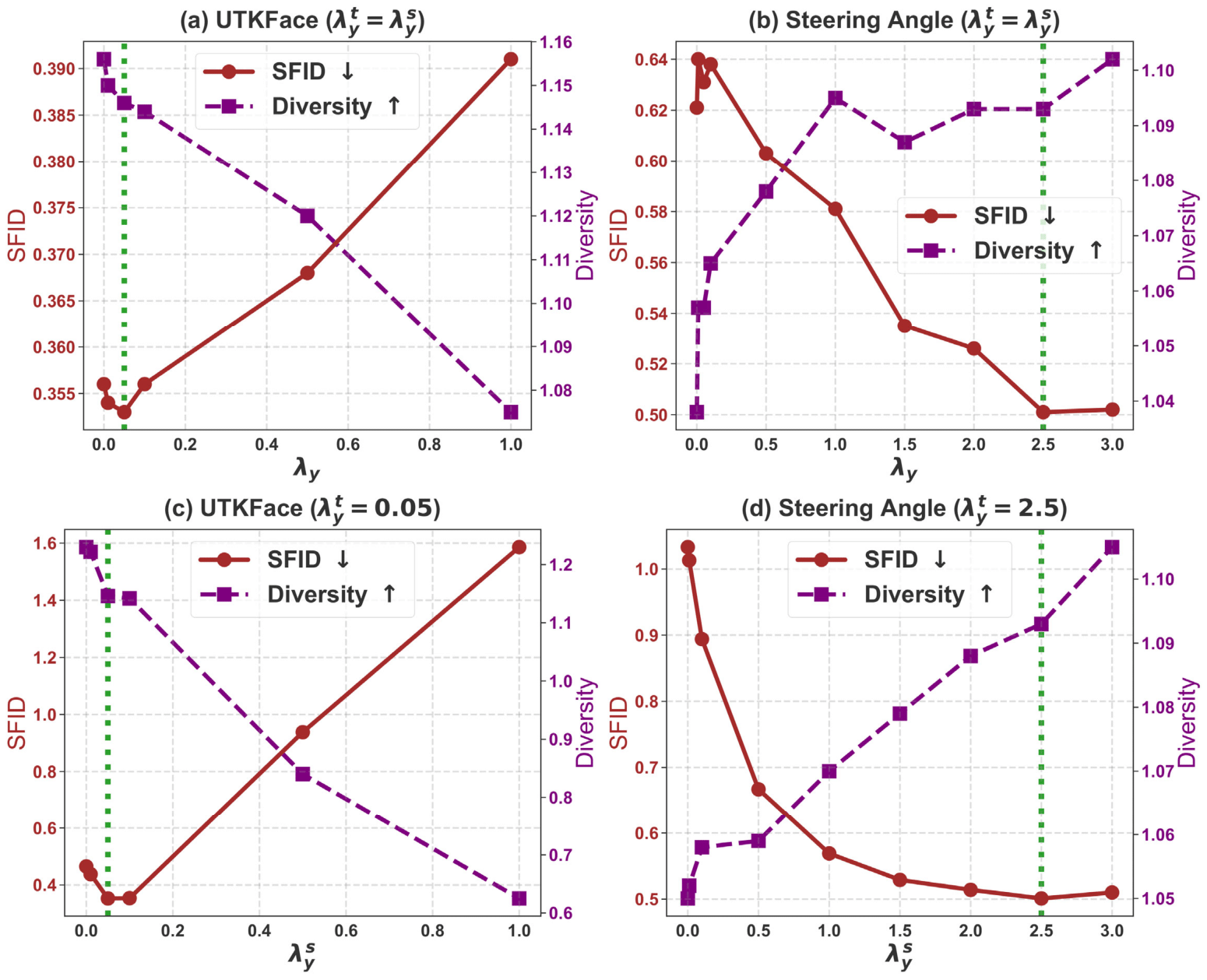}
		\caption{ Effect of the Weighting Coefficient $\lambda_y$ in Eq.~\eqref{eq:g_i} on iCCDM Across Two $64\times64$ Datasets. $\lambda_y^t$ and $\lambda_y^s$ denote the values of $\lambda_y$ used during training and sampling, respectively. }
		\label{fig:ab_effect_of_lambda_y}
	\end{figure}

    \begin{table}[!h]  
		\setlength{\tabcolsep}{1mm} 
		\footnotesize
		\centering
		\caption{ Effect of the $F_{\bm{\theta}}$ Network Architecture on iCCDM. }
			\begin{tabular}{c|c|cccc}
				\toprule
				\textbf{Dataset} & $F_{\bm{\theta}}$  & \begin{tabular}[c]{@{}c@{}} \textbf{SFID} \\ {(primary)} \end{tabular} & \textbf{NIQE} & \begin{tabular}[c]{@{}c@{}} \textbf{Diver} \\ \textbf{-sity} \end{tabular} & \begin{tabular}[c]{@{}c@{}} \textbf{Label} \\ \textbf{Score} \end{tabular} \\
				\midrule
								
				\multirow{3}[0]{*}{\begin{tabular}[c]{@{}c@{}} \begin{tabular}[c]{@{}c@{}} \textbf{UTKFace} \\ ($64\times 64$) \end{tabular} \end{tabular}} & CCDM & 0.353  & 1.561  & 1.146  & 5.788  \\
				& EDM & 0.351 & 1.644 & 1.223 & 6.611  \\
				& DiT   & 0.362  & 1.675  & 1.131  & 6.169  \\
				
				\cline{1-6}
				\specialrule{0em}{1pt}{1pt}
				
				\multirow{3}[0]{*}{\begin{tabular}[c]{@{}c@{}} \begin{tabular}[c]{@{}c@{}} \textbf{Steering Angle} \\ ($64\times 64$) \end{tabular} \end{tabular}} & CCDM & 0.501 & 1.700 & 1.093  & 4.891 \\	
				& EDM & 0.734 & 1.614  & 1.172  & 6.979  \\
				& DiT   & 0.669  & 1.614  & 1.077  & 5.677  \\
				
				\cline{1-6}
				\specialrule{0em}{1pt}{1pt}
				
				\multirow{2}[0]{*}{\begin{tabular}[c]{@{}c@{}} \begin{tabular}[c]{@{}c@{}} \textbf{UTKFace} \\ ($256\times 256$) \end{tabular} \end{tabular}} & CCDM &  0.230  & 1.389 & 1.177 & 6.599 \\ 
				&  DiT & 0.218 & 1.207 & 1.142 & 6.199 \\
				
				\cline{1-6}
				\specialrule{0em}{1pt}{1pt}
				
				\multirow{2}[0]{*}{\begin{tabular}[c]{@{}c@{}} \begin{tabular}[c]{@{}c@{}} \textbf{Steering Angle} \\ ($256\times 256$) \end{tabular} \end{tabular}} & CCDM & 0.763  & 2.549  & 1.202  & 13.780 \\	
				& DiT   & 0.582  & 1.390  & 0.981 & 4.743  \\
				
				\bottomrule
			\end{tabular}%
		\label{tab:ab_effect_network}%
	\end{table}%

   \begin{table}[!h]  
		\setlength{\tabcolsep}{1mm} 
		\footnotesize
		\centering
		\caption{ The Impact of Vicinity Type on iCCDM. FV and AV denote fixed and adaptive vicinities, respectively. }
			\begin{tabular}{c|c|cccc}
				\toprule
				\textbf{Dataset} & \textbf{Type}  & \begin{tabular}[c]{@{}c@{}} \textbf{SFID} \\ {(primary)} \end{tabular} & \textbf{NIQE} & \begin{tabular}[c]{@{}c@{}} \textbf{Diver} \\ \textbf{-sity} \end{tabular} & \begin{tabular}[c]{@{}c@{}} \textbf{Label} \\ \textbf{Score} \end{tabular} \\
				\midrule
				
				\multirow{5}[0]{*}{\begin{tabular}[c]{@{}c@{}} \begin{tabular}[c]{@{}c@{}} \textbf{UTKFace} \\ ($64\times 64$) \end{tabular} \end{tabular}} & Soft FV & 0.349  & 1.557  & 1.174  & 6.423  \\
				& Soft AV & 0.357  & 1.591  & 1.180  & 6.379  \\
				& Hybrid AV & 0.349  & 1.583  & 1.143  & 5.897  \\
				& Hard FV & 0.353  & 1.593  & 1.151  & 5.710  \\
				& Hard AV & 0.353  & 1.561  & 1.146  & 5.788  \\
				
				\cline{1-6}
				\specialrule{0em}{1pt}{1pt}
				
				\multirow{5}[0]{*}{\begin{tabular}[c]{@{}c@{}} \begin{tabular}[c]{@{}c@{}} \textbf{Steering Angle} \\ ($64\times 64$) \end{tabular} \end{tabular}} & Soft FV & 0.736 & 1.743  & 1.132  & 6.136  \\
				& Soft AV & 0.661 & 1.736  & 1.125  & 5.597  \\
				& Hybrid AV & 0.561  & 1.731  & 1.123 & 5.136  \\
				& Hard FV & 0.546  & 1.719  & 1.119  & 5.010  \\
				& Hard AV & 0.501 & 1.700 & 1.093  & 4.891 \\	
				
				\bottomrule
			\end{tabular}%
		\label{tab:ab_effect_vicinity_type}%
	\end{table}%

    \begin{table}[!h]  
		\setlength{\tabcolsep}{1mm} 
		\footnotesize
		\centering
		\caption{The Impact of Different Samplers on iCCDM. The sampling steps for SDE/ODE samplers are indicated after the hyphen. }
			\begin{tabular}{c|c|cccc}
				\toprule
				\textbf{Dataset} & \textbf{Sampler} & \begin{tabular}[c]{@{}c@{}} \textbf{SFID} \\ {(primary)} \end{tabular} & \textbf{NIQE} & \begin{tabular}[c]{@{}c@{}} \textbf{Diver} \\ \textbf{-sity} \end{tabular} & \begin{tabular}[c]{@{}c@{}} \textbf{Label} \\ \textbf{Score} \end{tabular} \\ 
				\midrule
				
				\multirow{6}[0]{*}{\begin{tabular}[c]{@{}c@{}} \begin{tabular}[c]{@{}c@{}} \textbf{UTKFace} \\ ($64\times 64$) \end{tabular} \end{tabular}} & SDE-10 & 0.470 & 1.813  & 1.262  & 7.645 \\
				& SDE-32 & 0.353  & 1.561  & 1.146  & 5.788  \\
				& SDE-50 & 0.368  & 1.828 & 1.197 & 6.405 \\
				\cdashline{2-6} \specialrule{0em}{1pt}{1pt}
				& ODE-10 & 0.406  & 1.750  & 1.242  & 7.528  \\
				& ODE-32 & 0.358  & 1.885  & 1.208  & 7.138   \\
				& ODE-50 & 0.356  & 1.885  & 1.205  & 7.119  \\
				
				\cline{1-6} \specialrule{0em}{1pt}{1pt}
				
				\multirow{6}[0]{*}{\begin{tabular}[c]{@{}c@{}} \textbf{Steering} \\ \textbf{Angle} \\ ($64\times 64$) \end{tabular}} & SDE-10 & 0.537 & 2.227  & 1.120 & 5.368  \\
				& SDE-32 & 0.501 & 1.700 & 1.093 & 4.891 \\
				& SDE-50 & 0.501 & 1.635 & 1.081 & 4.877  \\
				\cdashline{2-6} \specialrule{0em}{1pt}{1pt}
				& ODE-10 & 0.574  & 2.197 & 1.119 & 5.386  \\
				& ODE-32 & 0.570  & 1.651 & 1.111 & 5.477  \\
				& ODE-50 & 0.571  & 1.616 & 1.108 & 5.478  \\
				\bottomrule
			\end{tabular}%
		\label{tab:ab_effect_sampler}%
	\end{table}%

	\begin{table}[!h]  
		\setlength{\tabcolsep}{1mm} 
		\footnotesize
		\centering
		\caption{ The Impact of Covariance Embedding Network's Architecture on iCCDM Across Two $64\times 64$ Datasets. }
			\begin{tabular}{c|c|cccc|c}
				\toprule
				\textbf{Dataset} &  \textbf{Network}  & \begin{tabular}[c]{@{}c@{}} \textbf{SFID} \\ {(primary)} \end{tabular} & \textbf{NIQE} & \begin{tabular}[c]{@{}c@{}} \textbf{Diver} \\ \textbf{-sity} \end{tabular} & \begin{tabular}[c]{@{}c@{}} \textbf{Label} \\ \textbf{Score} \end{tabular} & \textbf{\# Params} \\
				\midrule
				
				\multirow{2}[0]{*}{\textbf{UTKFace}} & MLP5 & 0.353  & 1.614  & 1.147  & 5.742 & $6.14\times 10^7$ \\
				& CNN5 & 0.353  & 1.561  & 1.146  & 5.788 & $2.80\times 10^6$ \\
				
				\cline{1-7}
				\specialrule{0em}{1pt}{1pt}
				
				\multirow{2}[0]{*}{\begin{tabular}[c]{@{}c@{}} \textbf{Steering} \\ \textbf{Angle} \end{tabular}} & MLP5 & 0.498 & 1.676 & 1.092  & 5.301 & $6.14\times 10^7$ \\
				& CNN5 & 0.501 & 1.700 & 1.093  & 4.891 & $2.80\times 10^6$ \\
				
				\bottomrule
			\end{tabular}%
		\label{tab:ab_effect_y2cov_network}%
	\end{table}%

	\subsection{Experimental Setup}\label{sec:main_setup} 
	
	{\setlength{\parindent}{0cm}\textbf{Datasets.}} We compare the performance of iCCDM against other baselines on four benchmark datasets spanning resolutions from $64\times64$ to $256\times 256$: RC-49~\cite{ding2023ccgan}, UTKFace~\cite{utkface}, Steering Angle~\cite{steeringangle,steeringangle2}, and Cell-200~\cite{ding2023ccgan}. These datasets cover a diverse set of CCGM scenarios, including object pose control, age-conditioned face generation, autonomous driving scenes, and microscopic image synthesis.

    The RC-49 dataset contains 44,051 RGB images of 49 chair categories at $64\times64$ resolution, with each image annotated by a continuous yaw angle between $0.1^\circ$ and $89.9^\circ$. We construct a sparse training set by selecting angles with odd last digits and randomly sampling 25 images per angle, resulting in 11,250 training images spanning 450 distinct angles.
    
    UTKFace consists of 14,760 RGB face images labeled by age, ranging from 1 to 60 years. The number of images per age varies from 50 to over 1,000. All samples are used for training, and experiments are conducted at resolutions of $64\times 64$, $128\times 128$, $192\times 192$, and $256\times 256$.
    
    The Steering Angle dataset is derived from an autonomous driving dataset~\cite{steeringangle,steeringangle2} and preprocessed by Ding \emph{et al.}~\cite{ding2023ccgan}. It includes 12,271 RGB images captured from a dashboard-mounted camera, labeled with 1,774 continuous steering angles spanning $[-80^\circ, 80^\circ]$. We conduct experiments at $64\times 64$, $128\times 128$, and $256\times 256$ resolutions.
    
    Cell-200 is a synthetic microscopy dataset comprising 200,000 grayscale images at $64\times64$ resolution, each depicting a variable number of cells from 1 to 200. For training, we subsample images with odd-valued cell counts and randomly select 10 images per count, yielding 1,000 training samples. Evaluation is performed on the full dataset.
    
	{\setlength{\parindent}{0cm}\textbf{Compared Methods.}} We compare three families of generative models: (1) Text-to-image diffusion models, including SD v1.5~\cite{rombach2022high}, SD 3 Medium~\cite{esser2024scaling}, FLUX.1~\cite{labs2025flux1kontextflowmatching, flux2024, flux-2-2025}, and Qwen-Image~\cite{wu2025qwen}. (2) GAN-based CCGM models, namely CcGAN (SVDL+ILI)~\cite{ding2023ccgan}, Dual-NDA~\cite{ding2024turning}, and CcGAN-AVAR~\cite{ding2025imbalance}. (3) Diffusion-based CCGM models, including CcDPM~\cite{zhao2024ccdpm}, CCDM~\cite{ding2025ccdm}, and the proposed iCCDM.
	
	{\setlength{\parindent}{0cm}\textbf{Implementation Setup.}} The implementation settings for GAN-based CCGM models, CcDPM, and CCDM follow those reported in~\cite{ding2025ccdm}. For sampling with CcDPM and CCDM, we adopt DDIM~\cite{song2021denoising}, using 250 sampling steps for $64\times64$ experiments and 150 steps for experiments at higher resolutions.

    For text-to-image diffusion models, we employ Low-Rank Adaptation (LoRA)~\cite{hu2022lora} to fine-tune SD~3, FLUX.1, and Qwen-Image on all RGB datasets, while SD~v1.5 is fully fine-tuned. All fine-tuning is performed based on officially released checkpoints and the corresponding official fine-tuning scripts. Detailed experimental setups are provided in Appendix~G.
 
    
    For iCCDM, we use the CCDM UNet~\cite{ding2025ccdm} as the network backbone for all experiments except for Steering Angle ($256\times256$), where DiT~\cite{peebles2023scalable} is adopted due to its superior performance. Furthermore, for the Steering Angle experiments at $128\times128$ and $256\times256$, we set $\lambda_y=\lambda_y^t=0.01$ during training and $\lambda_y=\lambda_y^s=0.1$ during sampling. In all other experiments, $\lambda_y$ is kept the same for training and sampling, i.e., $\lambda_y^t=\lambda_y^s$. Detailed configurations are summarized in Appendix~F.
    
	{\setlength{\parindent}{0cm}\textbf{Evaluation Setup.}} Following established evaluation protocols~\cite{ding2021ccgan, ding2023ccgan, ding2023efficient, ding2024turning, ding2025ccdm}, we generate 179{,}800, 60{,}000, 100{,}000, and 200{,}000 samples per method on the RC-49, UTKFace, Steering Angle, and Cell-200 datasets, respectively. Performance is primarily assessed using the \textit{Sliding Fr'echet Inception Distance} (SFID)~\cite{ding2023ccgan}, computed as the mean FID across predefined evaluation centers. We additionally report the \textit{Naturalness Image Quality Evaluator} (NIQE)~\cite{mittal2012making}, Diversity~\cite{ding2023ccgan}, and Label Score~\cite{ding2023ccgan} as auxiliary metrics. The Diversity metric is not reported for Cell-200 due to the absence of categorical labels. Lower values indicate better performance for SFID, NIQE, and Label Score, whereas higher Diversity reflects greater sample variety. For details on the evaluation protocol, we refer readers to Appendix~S.V of CCDM~\cite{ding2025ccdm}.
	
	\subsection{Experimental Results}\label{sec:main_results}

	We evaluate the candidate methods across nine experimental settings, which include four datasets and four resolutions. The complete quantitative results are presented in Tables~\ref{tab:main_results_low}, \ref{tab:main_results_high}, and \ref{tab:ab_sampling_cost}, while selected visual results are displayed in Figs.~\ref{fig:scatter_sfid_vs_speed_sa64} and \ref{fig:example_imgs_real_and_fake}. These experimental results lead to the following key findings:  
    \begin{itemize}
        \item Overall, the proposed iCCDM outperforms other diffusion-based models. In particular, iCCDM surpasses CCDM in \textbf{eight out of nine settings} according to the primary metric SFID. The only exception is Cell-200, where iCCDM attains a slightly higher SFID than CCDM, while still exhibiting superior visual quality. Moreover, as shown in Table~\ref{tab:ab_sampling_cost}, iCCDM requires fewer sampling steps, resulting in a 4$\times$ faster sampling speed and lower memory consumption compared to CCDM on the Steering Angle dataset at $64\times64$ resolution.

        \item Despite their large model size and pretraining on massive datasets, state-of-the-art text-to-image diffusion models still struggle with CCGM tasks, as shown in Tables~\ref{tab:main_results_low} and~\ref{tab:main_results_high}. Notably, as illustrated in Fig.~\ref{fig:example_imgs_real_and_fake}, SD 3 and FLUX.1 fail to generate realistic chair images from given angles on the relatively simple RC-49 dataset. 

        \item In cross-model family comparisons, iCCDM outperforms vanilla CcGANs~\cite{ding2023ccgan, ding2024turning} by a substantial margin. Moreover, iCCDM surpasses the recently proposed CcGAN-AVAR in \textbf{seven out of nine settings} according to the SFID metric. The remaining two cases are RC-49 and UTKFace at $256\times256$ resolution, where iCCDM achieves a slightly higher SFID than CcGAN-AVAR, while still delivering superior visual quality.

        
    \end{itemize}


    \section{Ablation Study} \label{sec:ablation_study}

    In addition to the results above, we conduct six ablation studies on UTKFace and Steering Angle to evaluate the impact of key components and hyperparameters in iCCDM. These include the effects of hard adaptive vicinity and $y$-dependent diffusion (Table~\ref{tab:ablation_components}), SDE and ODE samplers with varying steps (Table~\ref{tab:ab_effect_sampler}), different backbones for $F_{\bm{\theta}}$ (Table~\ref{tab:ab_effect_network}), various vicinity types (Table~\ref{tab:ab_effect_vicinity_type}), CNN-based covariance embedding (Table~\ref{tab:ab_effect_y2cov_network}), and different $\lambda_y$ values (Fig.~\ref{fig:ab_effect_of_lambda_y}). These results validate our model design and configuration. Notably, Table~\ref{tab:ablation_components} shows that combining Hard AV and $y$-dependent diffusion yields the best performance, confirming our modification to EDM. Table~\ref{tab:ab_effect_sampler} supports the use of SDE-32 for sampling. Table~\ref{tab:ab_effect_network} highlights that CCDM UNet is best for low-resolution, while DiT excels in $256\times 256$ experiments. Table~\ref{tab:ab_effect_vicinity_type} shows hard AV is more stable, particularly for Steering Angle. Table~\ref{tab:ab_effect_y2cov_network} indicates that using CNN for covariance embedding reduces model size without significant performance loss. Finally, Fig.~\ref{fig:ab_effect_of_lambda_y} demonstrates that the optimal $\lambda_y$ value depends on the dataset, with our selection providing the best results.

	\section{Conclusion} \label{sec:conclusion}
	
	In this paper, we introduce the \textit{improved Continuous Conditional Diffusion Model} (iCCDM), which addresses the limitations of its predecessor, CCDM. By extending the EDM framework, iCCDM introduces condition-aware diffusion processes formulated using matrix-form SDEs and PF-ODEs. We also derive corresponding preconditioning, network backbone, vicinal training, and sampling procedures for the matrix-form setting. Additionally, iCCDM enhances efficiency by replacing the large covariance embedding network with a lightweight CNN, significantly reducing memory overhead. Extensive experiments across various benchmark datasets demonstrate iCCDM’s superior performance, establishing it as a powerful solution for quantitatively controllable image synthesis.

	\bibliographystyle{IEEEtran}
	\bibliography{./bibliography.bib}

	\vfill
	
	\clearpage
	\appendices
	
	\setcounter{figure}{0}
	\setcounter{table}{0}
	\setcounter{equation}{0}
	\renewcommand{\thefigure}{S.\arabic{figure}} 
	\renewcommand{\thetable}{S.\arabic{table}}   
	\renewcommand{\theequation}{S.\arabic{equation}} 
	\renewcommand{\thetheorem}{S.\arabic{theorem}} 
	\renewcommand{\thedefinition}{S.\arabic{definition}} 
	\renewcommand{\thelemma}{S.\arabic{lemma}} 
	\renewcommand{\theremark}{S.\arabic{remark}}
	\renewcommand{\thecorollary}{S.\arabic{corollary}} 
	

	\section{GitHub repository}\label{supp:codes}
	The source code and implementation details will be made publicly available at:
	\begin{center}
		\url{https://github.com/UBCDingXin/CCDM}
	\end{center}

	\section{Diffusion Expressed by It\^{o} SDEs and PF-ODE}

	Diffusion models convert real data into noise by progressively perturbing the data distribution $p_0$ with Gaussian noise, and they also transform noise back into data through a reverse denoising process. Song \emph{et al.}. \cite{song2021scorebased} model these forward and reverse processes using \textit{It\^{o} Stochastic Differential Equations} (It\^{o} SDEs) \cite{oksendal2013stochastic}. 
	
	The forward SDE is generally expressed as:
    \begingroup
	\[
	\mathrm{d}\bm{X}_t\ = \bm{f}(\bm{X}_t,t)\mathrm{d}t + \bm{G}(\bm{X}_t, t)\mathrm{d}\bm{B}_t,
	\label{eq:SDE_forward_general}
	\]
    \endgroup
	where $t\in[0,T]$ is a continuous time variable, $\bm{f}(\cdot,t):\mathbb{R}^d\rightarrow\mathbb{R}^d$ is the drift coefficient, $\bm{G}(\cdot,t):\mathbb{R}^d\rightarrow\mathbb{R}^{d\times d}$ is the diffusion coefficient, and $\bm{B}_t$ is a standard Wiener process. The forward process begins with the initial state $\bm{X}_0$ following the data distribution $p_0$ and ends with the final state $\bm{X}_T$ following a prior distribution $p_T$. The intermediate state $\bm{X}_t$ at time $t$ follows the marginal distribution $p_t$. 
	
	The corresponding reverse-time SDE, as derived by Song \emph{et al.}. \cite{song2021scorebased} based on results in \cite{oksendal2013stochastic}, is given by:
	\begingroup
    \[
	\mathrm{d}\bm{X}_t = & \left\{ \bm{f}(\bm{X}_t,t) - \nabla_{\bm{x}_t}\cdot\left[ \bm{G}(\bm{X}_t, t)\bm{G}(\bm{X}_t, t)^\intercal \right] \right. \\
	& \left. - \bm{G}(\bm{X}_t, t)\bm{G}(\bm{X}_t, t)^\intercal \nabla_{\bm{x}_t}\log p_t(\bm{X}_t) \right\}\mathrm{d}t \label{eq:SDE_reverse_general}\\
	& + \bm{G}(\bm{X}_t, t)\mathrm{d}\bar{\bm{B}_t},
	\]
     \endgroup
	where $\bar{\bm{B}}_t$ is a reverse-time standard Wiener process, and $\nabla_{\bm{x}_t}\log p_t(\bm{X}_t)$ is the score function.
	
	The \textit{Probability Flow Ordinal Differential Equation} (PF-ODE) to Eq.~\eqref{eq:SDE_forward_general} derived by \cite{song2021scorebased} has the following form:
	\begingroup
    \[
	\mathrm{d}\bm{X}_t = & \left\{ \bm{f}(\bm{X}_t,t) - \frac{1}{2}\nabla_{\bm{x}_t}\cdot\left[ \bm{G}(\bm{X}_t, t)\bm{G}(\bm{X}_t, t)^\intercal\right] \right.\\
	& \left. - \frac{1}{2}\bm{G}(\bm{X}_t, t)\bm{G}(\bm{X}_t, t)^\intercal\nabla_{\bm{x}_t}\log p_t(\bm{X}_t)   \right\}\mathrm{d}t.
	\label{eq:pfode_general}
	\]
	\endgroup
     
	Song \emph{et al.}. \cite{song2021scorebased} simplify the general formulations by assuming:
	\begingroup
     \[
	\bm{f}(\bm{X}_t,t) &= f(t)\bm{X}_t, \\
	\bm{G}(\bm{X}_t, t) &= g(t)\bm{I},
	\label{eq:song_assumption}
	\]
     \endgroup
	where $f(\cdot):\mathbb{R}\rightarrow\mathbb{R}$ and $g(\cdot):\mathbb{R}\rightarrow\mathbb{R}$ are scalar-valued functions, and $\bm{I}$ is the identity matrix. Under this simplification, the forward and reverse SDEs (Eqs.~\eqref{eq:SDE_forward_general} and \eqref{eq:SDE_reverse_general}) reduce respectively to:
	\begingroup
	\[
	\mathrm{d}\bm{X}_t\ &= f(t)\bm{X}_t\mathrm{d}t + g(t)\mathrm{d}\bm{B}_t, \label{eq:SDE_forward_simple}\\
	\mathrm{d}\bm{X}_t\ &= \left[ f(t)\bm{X}_t - g^2(t)\nabla_{\bm{x}_t}\log p_t(\bm{X}_t) \right]\mathrm{d}t  + g(t)\mathrm{d}\bar{\bm{B}_t}. \label{eq:SDE_reverse_simple}
	\]
	\endgroup
	The PF-ODE in Eq.~\eqref{eq:pfode_general} similarly simplifies to:
    \begingroup
     \[
	\mathrm{d}\bm{X}_t = \left[ f(t)\bm{X}_t - \frac{1}{2}g^2(t) \nabla_{\bm{x}_t}\log p_t(\bm{X}_t) \right] \mathrm{d}t.
	\label{eq:pfode_simple}
	\]
    \endgroup
 
	Song \emph{et al.}. \cite{song2021scorebased} have shown that when $f(t)=-\frac{1}{2}\beta(t)$ and $g(t)=\sqrt{\beta(t)}$, Eqs.~\eqref{eq:SDE_forward_simple} and \eqref{eq:SDE_reverse_simple} define the \textit{Variance Preserving SDE} (VP-SDE), where $\beta(t)$ is a real-valued function of time $t$ and the total variance remains bounded as $t$ increases. When $f(t)=0$ and $g(t)=\sqrt{\frac{\mathrm{d}\sigma^2(t)}{\mathrm{d}t}}$, the same equations describe the \textit{Variance Exploding SDE} (VE-SDE), in which $\sigma^2(t)$ increases monotonically with $t$ and the variance grows indefinitely over time, reflecting its “exploding” behavior.

	\section{Proof For Theorem \ref{thm:noising_process}}

	To prove Theorem \ref{thm:noising_process}, we formulate the conditional forward diffusion in a more general form:
     \begingroup
	\[
	\mathrm{d}\bm{X}_t = \bm{f}(\bm{X}_t,t)\mathrm{d}t + \bm{G}(t,y)\mathrm{d}\bm{B}_t,
	\label{eq:general_ito_sde_forward}
	\]
     \endgroup
	where $\bm{X}_t\in\mathbb{R}^d$, $\bm{f}(\bm{X}_t,t)\in\mathbb{R}^d$ denotes the drift coefficient, $\bm{G}(t,y)\in \mathbb{R}^{d\times d}$ is the diffusion coefficient depending on $y$, $\bm{B}_t$ is a standard Wiener process, and the initial state $\bm{X}_0$ follows the real data distribution $p_0$. Following \cite{song2021scorebased, karras2022elucidating}, the drift coefficient is simplified as
	\begingroup
    \[
	\bm{f}(\bm{X}_t,t) = f(t)\bm{X}_t,
	\label{eq:drift_simple}
	\]
    \endgroup
	where $f(t)\in\mathbb{R}$ is a real integrable function. \textbf{Unlike \cite{song2021scorebased, karras2022elucidating}, however, $\bm{G}(t,y)$ explicitly depends on the conditioning variable $y$, and yields a matrix rather than a scalar.} Furthermore, we denote
	\begingroup
    \[
	\bm{Q}(t,y)\triangleq \bm{G}(t,y)\bm{G}(t,y)^\intercal.
	\label{eq:Q_GG}
	\]
     \endgroup
	
	Then, Theorem \ref{thm:noising_process} can be seen as a special case of Theorem \ref{thm:noising_process_general}, when $f(t)\equiv0$.
	
	\begin{theorem}
		\label{thm:noising_process_general}
		Given the forward diffusion process defined in Eq.~\eqref{eq:general_ito_sde_forward}, the conditional distribution of $\bm{X}_t$ given $\bm{X}_0 = \bm{x}_0$ is Gaussian:
		\begingroup
        \[
		\bm{X}_t | \bm{X}_0 = \bm{x}_0 \sim \mathcal{N}(\alpha_t\bm{x}_0, \bm{\Sigma}(t,y)),
		\label{eq:add_noise_eq1}
		\]
        \endgroup
		where 
		\begingroup
        \[
		\alpha_t = \exp\left( \int_0^t f(s)\mathrm{d}s \right),
		\]
        \endgroup
		and the covariance matrix is given by
		\begingroup
        \[
		\bm{\Sigma}(t,y) = \int_0^t \left( \frac{\alpha_t}{\alpha_s} \right)^2 \bm{Q}(s,y)\mathrm{d}s.
		\]
        \endgroup
	\end{theorem}
	
	\begin{proof}
		
		We apply the method of integrating factor \cite{oksendal2013stochastic} to derive the solution path of the forward diffusion process. Let
		\begingroup
        \[
		\Phi(t) = \exp\left(-\int_0^t f(s)\mathrm{d}s \right),
		\]
        \endgroup
		then we have 
        \begingroup
		\[
		\frac{\mathrm{d}}{\mathrm{d}t}\Phi(t)=-f(t)\Phi(t)
		\label{eq:derivative_Phi}
		\]
        \endgroup
		with $\Phi(0)=1$. Defining $\bm{Y}_t=\Phi(t)\bm{X}_t$ and applying It\^{o}'s lemma \cite{oksendal2013stochastic}, we obtain
		\begingroup
        \[
		\mathrm{d}\bm{Y}_t = \Phi(t)\mathrm{d}\bm{X}_t + \bm{X}_t\mathrm{d}\Phi(t) + \mathrm{d}\langle \Phi(\cdot), \bm{X}_\cdot \rangle_t.
		\label{eq:Yt_SDE}
		\]
        \endgroup
		Since $\Phi(t)$ is deterministic, the quadratic covariation term $\mathrm{d}\langle \Phi(\cdot), \bm{X}_\cdot \rangle_t=0$. Substituting Eqs.~\eqref{eq:general_ito_sde_forward} and \eqref{eq:derivative_Phi} into Eq.~\eqref{eq:Yt_SDE} yields
		\begingroup
        \[
		\mathrm{d}\bm{Y}_t & = \Phi(t)\left( f(t)\bm{X}_t\mathrm{d}t + \bm{G}(t,y)\mathrm{d}\bm{B}_t \right) - f(t)\Phi(t)\bm{X}_t\mathrm{d}t \\
		& = \Phi(t)\bm{G}(t,y)\mathrm{d}\bm{B}_t.
		\label{eq:Yt_SDE_simple}
		\]
        \endgroup
		Using $\Phi(0)=1$, we have
        \begingroup
		\[
		\bm{Y}_t &= \bm{Y}_0 + \int_0^t \Phi(s)\bm{G}(s,y)\mathrm{d}\bm{B}_s \\
		&= \bm{X}_0 + \int_0^t \Phi(s)\bm{G}(s,y)\mathrm{d}\bm{B}_s.
		\label{eq:solution_Yt_SDE}
		\]
        \endgroup
		Multiplying both sides of Eq.~\eqref{eq:solution_Yt_SDE} by
        \begingroup
		\[
		\Phi(t)^{-1} = \exp\left( \int_0^t f(s)\mathrm{d}s \right),
		\]
        \endgroup
		gives
        \begingroup
		\[
		\bm{X}_t &= \Phi(t)^{-1}\bm{Y}_t \\
		&= \Phi(t)^{-1}\bm{X}_0 + \Phi(t)^{-1}\int_0^t \Phi(s)\bm{G}(s,y)\mathrm{d}\bm{B}_s
		\]
        \endgroup
		Since $\Phi(t)^{-1}$ is deterministic, it can be moved inside the integral:
		\begingroup
        \[
		\bm{X}_t & = \Phi(t)^{-1}\bm{X}_0  + \int_0^t \frac{\Phi(t)^{-1}}{\Phi(s)^{-1}}\bm{G}(s,y)\mathrm{d}\bm{B}_s \\
		& = \alpha_t\bm{X}_0  + \int_0^t \frac{\alpha_t}{\alpha_s}\bm{G}(s,y)\mathrm{d}\bm{B}_s,
		\label{eq:Xt_SDE_solution}
		\]
        \endgroup
		where $\alpha_t = \Phi(t)^{-1}$. As $\frac{\alpha_t}{\alpha_s}\bm{G}(s,y)$ is deterministic, the Itô integral in Eq.~\eqref{eq:Xt_SDE_solution} is Gaussian with mean $\bm{0}$ and covariance matrix
		\begingroup
        \[
		\bm{\Sigma}(t,y) = \int_0^t \left( \frac{\alpha_t}{\alpha_s} \right) \bm{Q}(s,y) \mathrm{d}s.
		\]
        \endgroup
		
		Therefore, conditional on $\bm{X}_0=\bm{x}_0$, 
		\begingroup
        \[
		\bm{X}_t | \bm{X}_0 = \bm{x}_0 \sim \mathcal{N}(\alpha_t\bm{x}_0, \bm{\Sigma}(t,y)),
		\]
        \endgroup
		
	\end{proof}
	
	\begin{corollary}[Nosing Perturbation] By reparameterization, Eq.\eqref{eq:add_noise_eq1} can be rewritten as
		\begingroup
        \[
		\bm{X}_t = \alpha_t \bm{x}_0 + \bm{\Sigma}(t,y)^{\frac{1}{2}}\bm{\varepsilon}, \quad \bm{\varepsilon}\sim\mathcal{N}(\bm{0},\bm{I}), 
		\label{eq:add_noise_eq2}
		\]
        \endgroup
		where $\bm{\Sigma}(t,y)^{\frac{1}{2}}$ satisfies $\bm{\Sigma}(t,y)^{\frac{1}{2}}\bm{\Sigma}(t,y)^{\frac{1}{2}}=\bm{\Sigma}(t,y)$.
	\end{corollary}
	
	
	
	

	\section{Deriving Vicinal Score Estimate}\label{supp:derive_vicinal_score}

    {\setlength{\parindent}{0cm} \textit{(1) Deriving Eq.~\eqref{eq:vicinal_denoiser_loss}:} } 
    \begingroup
	\small
	\[
		& \mathcal{L}(D;y,\bm{\Sigma}) \\
		= &\mathbb{E}_{\bm{X}\sim p_\text{data}(\bm{x}|y)}\mathbb{E}_{\tilde{\bm{X}}\sim\mathcal{N}(\bm{X},\bm{\Sigma})}\left[ \left\| D(\tilde{\bm{X}},y,\bm{\Sigma}) - \bm{X} \right\|_2^2  \right]\\
		= & \mathbb{E}_{\bm{X}\sim p_\text{data}(\bm{x}|y)} \left[ \int \mathcal{N}(\tilde{\bm{x}}; \bm{X}, \bm{\Sigma}) \left\| D(\tilde{\bm{x}},y,\bm{\Sigma}) - \bm{X} \right\|_2^2   \mathrm{d} \tilde{\bm{x}} \right] \\
        & \text{(Using the property of the Dirac delta function.)} \\
		\approx & C_y\sum_{i=1}^{N} \left[ W_{i,y} \int \mathcal{N}(\tilde{\bm{x}}; \bm{x}_{0,i}, \bm{\Sigma}) \left\| D(\tilde{\bm{x}},y,\bm{\Sigma}) - \bm{x}_{0,i} \right\|_2^2   \mathrm{d} \tilde{\bm{x}} \right] \\
		= & C_y \int \underbrace{\left[ \sum_{i=1}^{N}  W_{i,y} \mathcal{N}(\tilde{\bm{x}}; \bm{x}_{0,i}, \bm{\Sigma}) \left\| D(\tilde{\bm{x}},y,\bm{\Sigma}) - \bm{x}_{0,i} \right\|_2^2 \right]}_{\triangleq \hat{\mathcal{L}}^v(D; \tilde{\bm{x}}, y,\bm{\Sigma})} \mathrm{d} \tilde{\bm{x}} \\
		\triangleq & \hat{\mathcal{L}}^v(D;y,\bm{\Sigma}) 
	\]
	\endgroup

    {\setlength{\parindent}{0cm} \textit{(2) Deriving Eq.~\eqref{eq:D_star}:} } 
    
    As pointed out by Karras et~al.~\cite{karras2022elucidating}, Eq.~\eqref{eq:minimization_D_ast} is a convex optimization problem with respect to $D$. Thus, we can derive $D^\ast$ in Eq.~\eqref{eq:D_star} by setting the gradient of $\hat{\mathcal{L}}^v(D; \tilde{\bm{x}}, y,\bm{\Sigma})$ with respect to $D$ to zero:
	\begingroup
	\small
	\[
		& \bm{0} = \nabla_{D}\hat{\mathcal{L}}^v(D; \tilde{\bm{x}}, y,\bm{\Sigma}) \\
		\Rightarrow & \bm{0} = \nabla_{D}\left[ \sum_{i=1}^{N} W_{i,y} \mathcal{N}(\tilde{\bm{x}}; \bm{x}_{0,i}, \bm{\Sigma}) \left\| D(\tilde{\bm{x}},y,\bm{\Sigma}) - \bm{x}_{0,i} \right\|_2^2 \right] \\
		\Rightarrow & \bm{0} = \sum_{i=1}^{N} W_{i,y}\mathcal{N}(\tilde{\bm{x}}; \bm{x}_{0,i}, \bm{\Sigma}) \nabla_{D}\left[ \left\| D(\tilde{\bm{x}},y,\bm{\Sigma}) - \bm{x}_{0,i} \right\|_2^2 \right] \\
		\Rightarrow & \bm{0} = \sum_{i=1}^{N} W_{i,y}\mathcal{N}(\tilde{\bm{x}}; \bm{x}_{0,i}, \bm{\Sigma})\left[ 2D(\tilde{\bm{x}},y,\bm{\Sigma}) - 2\bm{x}_{0,i} \right] \\
		\Rightarrow & D^\ast(\tilde{\bm{x}},y,\bm{\Sigma}) = \frac{\sum_{i=1}^{N} W_{i,y}\mathcal{N}(\tilde{\bm{x}}; \bm{x}_{0,i}, \bm{\Sigma})\bm{x}_{0,i}}{\sum_{i=1}^{N} W_{i,y}\mathcal{N}(\tilde{\bm{x}}; \bm{x}_{0,i}, \bm{\Sigma})}
	\]
	\endgroup

    {\setlength{\parindent}{0cm} \textit{(3) Deriving Eq.~\eqref{eq:p_sigma_vic_est}:} } 
    \begingroup
	\small
    \[
	p_{\bm{\Sigma}}(\tilde{\bm{x}}|y) 
	= & \left[ p_\text{data} \ast \mathcal{N}(\bm{x}_0, \bm{\Sigma}) \right](\tilde{\bm{x}}|y) \\
	\approx & \int \hat{p}^v_\text{data}(\bm{x}_0|y) \mathcal{N}(\tilde{\bm{x}} ; \bm{x}_0, \bm{\Sigma}) \mathrm{d} \bm{x}_0 \\
	= & \int \left[ C_y\sum_{i=1}^{N}W_{i,y}\delta(\bm{x}_0-\bm{x}_{0,i}) \right] \mathcal{N}(\tilde{\bm{x}} ; \bm{x}_0, \bm{\Sigma})  \mathrm{d} \bm{x}_0  \\
	= & C_y\sum_{i=1}^{N} W_{i,y} \int \mathcal{N}(\tilde{\bm{x}} ; \bm{x}_0, \bm{\Sigma}) \delta(\bm{x}_0-\bm{x}_{0,i}) \mathrm{d} \bm{x}_0 \\
	= & C_y\sum_{i=1}^{N} W_{i,y} \mathcal{N}(\tilde{\bm{x}} ; \bm{x}_{0,i}, \bm{\Sigma}) \triangleq   \hat{p}^v_{\bm{\Sigma}}(\tilde{\bm{x}}|y). 
	\]
     \endgroup

    {\setlength{\parindent}{0cm} \textit{(4) Deriving Eq.~\eqref{eq:grad_p_sigma_vic_est}:} } 
    \begingroup
	\small
    \[
	& \nabla_{\tilde{\bm{x}}}\log\hat{p}^v_{\bm{\Sigma}}(\tilde{\bm{x}}|y) \\
	= & \frac{\sum_{i=1}^{N} W_{i,y} \nabla_{\tilde{\bm{x}}}\mathcal{N}(\tilde{\bm{x}} ; \bm{x}_{0,i}, \bm{\Sigma})}{\sum_{i=1}^{N} W_{i,y} \mathcal{N}(\tilde{\bm{x}} ; \bm{x}_{0,i}, \bm{\Sigma})}. \\
	= & \frac{\sum_{i=1}^{N} W_{i,y} \mathcal{N}(\tilde{\bm{x}} ; \bm{x}_{0,i}, \bm{\Sigma})\bm{\Sigma}^{-1}(\bm{x}_{0,i}-\tilde{\bm{x}})}{\sum_{i=1}^{N} W_{i,y} \mathcal{N}(\tilde{\bm{x}} ; \bm{x}_{0,i}, \bm{\Sigma})} \\
	= & \bm{\Sigma}^{-1}\left( \frac{\sum_{i=1}^{N} W_{i,y} \mathcal{N}(\tilde{\bm{x}} ; \bm{x}_{0,i}, \bm{\Sigma})\bm{x}_{0,i}}{\sum_{i=1}^{N} W_{i,y} \mathcal{N}(\tilde{\bm{x}} ; \bm{x}_{0,i}, \bm{\Sigma})} - \tilde{\bm{x}}  \right).
	\label{eq:supp_grad_p_sigma_vic_est}
	\]
    \endgroup

	\section{Deriving Vicinal Training Loss}\label{supp:derive_vicinal_loss}

    The complete derivation of Eq.~\eqref{eq:hat_L_theta_Sigma} is shown as follows:
	\begingroup
	\small
	\[
		&\hat{\mathcal{L}}(\bm{\theta};\bm{\Sigma}) \\
		= & \mathbb{E}_{\substack{(\bm{X},Y)\sim \hat{p}^v_\text{data}(\bm{x},y) \\ \bm{\varepsilon}\sim\mathcal{N}(\bm{0},\bm{\Sigma})}}\left[ \left\| \Lambda_{\bm{\Sigma}}^{\frac{1}{2}} \left( D_{\bm{\theta}}(\bm{X}+\bm{\varepsilon},y,\bm{\Sigma}) - \bm{X} \right) \right\|_2^2 \right] \\
		= &  \int \mathbb{E}_{\bm{\varepsilon}\sim\mathcal{N}(\bm{0},\bm{\Sigma})} \left\| \Lambda_{\bm{\Sigma}}^{\frac{1}{2}} \left( D_{\bm{\theta}}(\bm{x}+\bm{\varepsilon},y,\bm{\Sigma}) - \bm{x} \right) \right\|_2^2 \\
		& \qquad\qquad\qquad \cdot \left[ \frac{1}{N}\sum_{j=1}^N\exp\left( -\frac{(y-y_j)^2}{2\sigma_\text{KDE}^2} \right) \right] \\
		& \qquad\qquad\qquad \cdot \left[ C_y \sum_{i=1}^{N}W_{i,y}\delta(\bm{x}-\bm{x}_{0,i})  \right] \mathrm{d}\bm{x}\mathrm{d}y \\
		= & \frac{C}{N}\sum_{i=1}^N\sum_{j=1}^N \mathbb{E}_{\substack{ \bm{\varepsilon}\sim\mathcal{N}(\bm{0},\bm{\Sigma}) \\ \eta\sim\mathcal{N}(0,\sigma_\text{KDE}^2) }} W_{i,y_j+\eta}\left\| \Lambda_{\bm{\Sigma}}^{\frac{1}{2}} \right. \\
		& \qquad\qquad\qquad  \left. \vphantom{\Lambda_{\bm{\Sigma}}^{\frac{1}{2}}} \cdot \left( D_{\bm{\theta}}(\bm{x}_i+\bm{\varepsilon},y_j+\eta,\bm{\Sigma}) - \bm{x}_i \right) \right\|_2^2
	\]
	\endgroup
	where $C$ absorbs all normalizing constants, and $\eta\triangleq y-y_j$.

	\section{Experimental Details}\label{supp:exp_details}

    The main experimental settings for implementing iCCDM are summarized in Table~\ref{tab:train_setup_iccdm}. EDM-related parameters introduced in~\cite{karras2022elucidating} and kept unchanged in our experiments are set as follows: $\sigma_\text{min}=0.002$, $\sigma_\text{max}=80$, $\rho=7.0$, $P_\text{mean}=-1.2$, $P_\text{std}=1.2$, $S_\text{churn}=80$, $S_\text{tmin}=0.05$, $S_\text{tmax}=50$, and $S_\text{noise}=1.003$. Sampling speed and memory consumption are evaluated on an Ubuntu server equipped with NVIDIA RTX 4090D-48G GPUs.
 
	\begin{table}[!h] 
		\setlength{\tabcolsep}{1mm} 
		\centering
		\caption{ Implementation Details of iCCDM.}
			\begin{tabular}{cl}
				\toprule
				\textbf{Dataset} &  \textbf{Setup} \\
				\midrule
				
				\begin{tabular}[c]{@{}c@{}} RC-49 \\ {(64$\times$64)}\end{tabular} & \begin{tabular}[c]{@{}l@{}} CCDM UNet, steps=100K, lr=1e-4, bs=128 \\ Hard AV, $N_\text{AV}=50$, $\lambda_y^t=\lambda_y^s=0.001$, \\ CFG's $\gamma=1.2$, SDE-32 \end{tabular} \\
				\midrule

				\begin{tabular}[c]{@{}c@{}} Cell-200 \\ {(64$\times$64)}\end{tabular} & \begin{tabular}[c]{@{}l@{}} CCDM UNet, steps=50K, lr=5e-5, bs=64 \\ Hard AV, $N_\text{AV}=20$, $\lambda_y^t=\lambda_y^s=0.01$, \\ CFG's $\gamma=1.5$, SDE-32 \end{tabular} \\			
				\midrule

				\begin{tabular}[c]{@{}c@{}} UTKFace \\ {(64$\times$64)}\end{tabular} & \begin{tabular}[c]{@{}l@{}} CCDM UNet, steps=100K, lr=1e-4, bs=128 \\ Hard AV, $N_\text{AV}=400$, $\lambda_y^t=\lambda_y^s=0.05$, \\ CFG's $\gamma=1.5$, SDE-32 \end{tabular} \\			
				\midrule
				
				\begin{tabular}[c]{@{}c@{}} UTKFace \\ {(128$\times$128)}\end{tabular} & \begin{tabular}[c]{@{}l@{}} CCDM UNet, steps=200K, lr=1e-5, bs=128 \\ Hard AV, $N_\text{AV}=400$, $\lambda_y^t=\lambda_y^s=0.01$, \\ CFG's $\gamma=1.5$, SDE-32 \end{tabular} \\			
				\midrule

				\begin{tabular}[c]{@{}c@{}} UTKFace \\ {(192$\times$192)}\end{tabular} & \begin{tabular}[c]{@{}l@{}} CCDM UNet, steps=800K, lr=1e-5, bs=112 \\ Hard AV, $N_\text{AV}=400$, $\lambda_y^t=\lambda_y^s=0.01$, \\ CFG's $\gamma=1.5$, SDE-32 \end{tabular} \\			
				\midrule

				\begin{tabular}[c]{@{}c@{}} UTKFace \\ {(256$\times$256)}\end{tabular} & \begin{tabular}[c]{@{}l@{}} DiT-B/4, steps=800K, lr=1e-5, bs=32 \\ Hard AV, $N_\text{AV}=400$, $\lambda_y^t=\lambda_y^s=0.01$, \\ CFG's $\gamma=1.5$, SDE-32 \end{tabular} \\			
				\midrule

				\begin{tabular}[c]{@{}c@{}} Steering \\ Angle \\ {(64$\times$64)}\end{tabular} & \begin{tabular}[c]{@{}l@{}} CCDM UNet, steps=100K, lr=1e-4, bs=128 \\ Hard AV, $N_\text{AV}=10$, $\lambda_y^t=\lambda_y^s=2.5$, \\ CFG's $\gamma=1.5$, SDE-32 \end{tabular} \\			
				\midrule
				
				\begin{tabular}[c]{@{}c@{}} Steering \\ Angle \\ {(128$\times$128)}\end{tabular} & \begin{tabular}[c]{@{}l@{}} CCDM UNet, steps=400K, lr=5e-5, bs=112 \\ Hard AV, $N_\text{AV}=10$, $\lambda_y^t=0.01$, \\ $\lambda_y^s=0.1$, CFG's $\gamma=1.5$, SDE-32 \end{tabular} \\			
				\midrule

				\begin{tabular}[c]{@{}c@{}} Steering \\ Angle \\ {(256$\times$256)}\end{tabular} & \begin{tabular}[c]{@{}l@{}} DiT-B/4, steps=400K, lr=1e-5, bs=36 \\ Hard AV, $N_\text{AV}=20$, $\lambda_y^t=0.01$, \\ $\lambda_y^s=0.1$, CFG's $\gamma=2.0$, SDE-32 \end{tabular} \\			
    
                \bottomrule
			\end{tabular}%
		\label{tab:train_setup_iccdm}%
	\end{table}%

	\section{Implementation Setups of Text-to-Image Models}\label{supp:comp_with_T2I}
	
	We fine-tuned four text-to-image diffusion models: (i) SD v1.5 with \textbf{full fine-tuning} of the denoising UNet under the DDPM formulation, and (ii) SD 3 Medium, Qwen-Image, and FLUX.1 with \textbf{LoRA} adapters injected into their transformer-based denoisers. For SD v1.5, training follows latent diffusion noise prediction: an image is encoded into latents by the VAE, Gaussian noise is added at a sampled timestep, and the UNet is optimized with an MSE objective to predict the injected noise (with optional timestep reweighting, such as Min-SNR). For SD 3/Qwen/FLUX, training follows a Flow-Matching setup: noisy inputs are formed by linear interpolation between clean latents and noise using the scheduler's $\sigma$, and the model is trained with a possibly $\sigma$-weighted MSE objective to match the velocity field target $v = \epsilon - x_0$ (SD 3 also supports output preconditioning to predict $x_0$). In all cases, we trained for approximately \textbf{two epochs} on the available data, which we considered a reasonable fine-tuning budget.

    We used fixed prompt templates: UTKFace uses ``a portrait of the face of [age] year old.'', SteeringAngle uses ``a road view with steering angle of [angle] degrees.'', and RC-49 uses ``a photo of a chair at yaw angle [yaw] degrees.'' Sampling is performed by loading the corresponding base pipeline, swapping in a full UNet checkpoint (SD v1.5) or attaching LoRA weights (SD 3/Qwen/FLUX), followed by conditional generation under classifier-free guidance variants (standard CFG for SD v1.5/SD 3/Qwen, and guidance embedding for FLUX). For computing evaluation metrics, we use a default directory layout for per-label PNG outputs and an evaluation-oriented export that produces PNGs for NIQE and H5 containers for SFID/Diversity/Label Score. To ensure engineering stability on large backbones, Qwen and FLUX run in \texttt{torch.bfloat16} to avoid OOM errors. Qwen further uses a \textbf{two-stage loading strategy}, where all text embeddings are precomputed and indexed during training/sampling, reducing peak VRAM usage by about \textbf{16 GiB}.



\end{document}